\documentclass[11pt]{article}

\usepackage[preprint]{acl}

\usepackage{times}
\usepackage{latexsym}
\usepackage{amsmath}
\usepackage{mathtools}
\usepackage{amsmath,amssymb,amsthm}
\usepackage{multicol}
    
\usepackage{amssymb}
\usepackage{booktabs}
\usepackage{multirow}
\usepackage{enumitem}
\usepackage{cleveref}
\usepackage{subcaption}
\usepackage{bbm}
\usepackage{stfloats}
\usepackage[T1]{fontenc}
\usepackage[utf8]{inputenc}
\usepackage{microtype}
\usepackage{inconsolata}
\usepackage{graphicx}
\usepackage{subcaption}
\usepackage{tikz}
\usepackage{pgfplots}
\usepgfplotslibrary{groupplots}
\pgfplotsset{compat=1.18}
\usepackage{CJKutf8}
\usepackage{xcolor}

\newcommand{\boxtitle}{\scriptsize\bfseries}
\newcommand{\boxsrc}{\tiny\color{black!55}}

\usepackage{amsmath,amssymb,amsthm}
\usepackage{xcolor}
\usepackage{subcaption}
\usepackage{graphicx}
\usepackage{tikz}
\usetikzlibrary{decorations.pathreplacing, arrows.meta,positioning}
\usepackage[most]{tcolorbox}

\newtheorem{proposition}{Proposition}

\newcommand{\valley}{\mathrm{valley}}
\newcommand{\peak}{\mathrm{peak}}
\newcommand{\Valley}{\mathrm{Valley}}
\newcommand{\Peak}{\mathrm{Peak}}
\newcommand{\pos}{\mathrm{pos}}
\newcommand{\nneg}{\mathrm{neg}}
\newcommand{\Pos}{\mathrm{Pos}}
\newcommand{\Neg}{\mathrm{Neg}}

\definecolor{vfBlue}{HTML}{0072B2}
\definecolor{vfOrange}{HTML}{E69F00}
\definecolor{vfGreen}{HTML}{009E73}
\definecolor{vfPurple}{HTML}{CC79A7}
\definecolor{vfRed}{HTML}{D55E00}
\definecolor{vfSky}{HTML}{56B4E9}
\definecolor{vfGray}{HTML}{6B7280}

\tikzset{
    token/.style={circle, inner sep=1.5pt},
    fav/.style={token, fill=vfBlue, draw=vfBlue!70!black},
    chal/.style={token, fill=vfOrange, draw=vfOrange!70!black},
    axis/.style={->, draw=vfGray},
    thresh/.style={densely dotted, very thick, draw=vfRed}
}

\newtcolorbox{rolloutbox}[2][]{
  enhanced,
  breakable,
  colback=#2!4,
  colframe=#2!55!black,
  boxrule=0.4pt,
  arc=1mm,
  left=1mm,
  right=1mm,
  top=1mm,
  bottom=1mm,
  fonttitle=\bfseries\footnotesize,
  fontupper=\footnotesize,
  #1
}

\title{A Gradient Perspective on RLVR Stability and\\ Winner Advantage Policy Optimization}

\author{Prasanth YSS\thanks{Equal contribution. PY: theory and experiments, ZR: engineering and experiments.}, Zhichen Ren\footnotemark[1], Rasa Hosseinzadeh, Ilan Gofman, Yuqi Chen, \\ \textbf{Zhaoyan Liu, Guangwei Yu, Jesse C. Cresswell, Satya Krishna Gorti} \\
       Layer 6 AI}

\newcommand{\qwen}{\text{Qwen3-4B}}
\newcommand{\smol}{\text{SmolLM3-3B}}
\newcommand{\gemma}{\text{Gemma3-4B}}

\newcommand{\ott}{\text{OTT-QA}}
\newcommand{\hotpot}{\text{Hotpot-QA}}

\newcommand{\datamath}{\text{Math-500}}
\newcommand{\numina}{\text{NuminaMath-LEAN}}
\newcommand{\wikidata}{\text{2-wiki}}
\newcommand{\aime}{\text{AIME'25}}

\newcommand{\method}{\text{WAPO}} 
\newcommand{\GRPO}{\text{GRPO}}
\newcommand{\GSPO}{\text{GSPO}}
\newcommand{\DAPO}{\text{DAPO}}

\begin{document}
\maketitle
\begin{abstract}
\looseness=-1
Reinforcement learning with verifiable rewards (RLVR) improves language-model reasoning, but \GRPO{}-style optimization remains prone to collapse. We analyse this instability through token-level gradient dynamics, deriving a taxonomy that predicts how updates affect next-token probabilities and entropy. The taxonomy shows that stability depends jointly on the advantage sign and token distribution under the current policy. Motivated by this finding, we propose Winner Advantage Policy Optimization (\method{}), a simple online clipped policy-gradient objective that updates only on positive-advantage completions. Across mathematical reasoning and multi-hop QA benchmarks, \method{} improves training stability and matches or outperforms baselines across multiple model families. Full code can be found at \href{https://github.com/layer6ai-labs/wapo}{https://github.com/layer6ai-labs/wapo}.
\end{abstract}

\section{Introduction}
\label{sec:introduction}
Reinforcement learning with verifiable rewards (RLVR) \citep{pyatkingeneralizing} has become a central recipe for improving language-model reasoning and planning \citep{wang2025kimina, he2025skywork}. \GRPO{} \citep{shao2024deepseekmath} and its variants \citep{yu2025dapo, liuunderstanding, cispo, zheng2025group, deng2025grpo, qi2026rethinking} use off-policy samples as introduced by \citet{PPO}. Rollouts generated by the old policy model $\pi_{\theta_\text{old}}$ are used for several gradient updates of the current policy model $\pi_{\theta}$ for sample efficiency. However, off-policy optimization methods remain fragile and susceptible to collapse (see Figure \ref{fig:collapse_examples}). 

\begin{figure}[t]                                                                  
  \centering                                                                               
  \begin{tikzpicture}[                                  
    basebox/.style={                                                          
      rounded corners=2pt,                                       
      inner sep=6pt,                               
      text width=0.94\columnwidth,                                         
      align=left,                                     
      font=\scriptsize                                                       
    },
    highbox/.style={
      basebox,
      fill=red!5,      
      draw=red!40      
    },
    lowbox/.style={
      basebox,
      fill=blue!5,     
      draw=blue!40     
    }
  ]

  \node[highbox] (high) {
    {\boxtitle High-entropy collapse}\\[-1pt]
    {\boxsrc Math, SmolLM3-3B, step 150}\\[2pt]
    {\ttfamily
    <think> ))\^{}x28**(\{\}) orderedList x231(color872426).txt \ldots{}\\
    Japanese/Arabic/CJK fragments \ldots{} identifier BERKELEY uc \ldots{}\\
    OS\begin{CJK*}{UTF8}{gbsn}受\end{CJK*} \ldots{} PEER assembly \ldots{} $\star\star\star\star\star$ \ldots{} server \ldots{}
    }\\[-1pt]
    \(\Rightarrow\) output spreads into task-irrelevant multilingual, code-like, and malformed text.
  };

  \node[lowbox, below=3pt of high] (low) {
    {\boxtitle Low-entropy collapse}\\[-1pt]
    {\boxsrc Math, Qwen-3-4B, step 139}\\[2pt]
    {\ttfamily
    <think>\\
    </think>\\
    Answer: <your answer>\\
    100000000000000000000000000000000000000000000000000000000000\\
    \ldots{}
    }\\[-1pt]
    \(\Rightarrow\) output contracts to a repetitive, malformed answer pattern.
  };

  \end{tikzpicture}
  \vspace{-6pt}
  \caption{Examples of collapsed generations during RLVR training. High-entropy collapse produces diverse task-irrelevant continuations, while low-entropy collapse degenerates into repetitive malformed text.}
  \label{fig:collapse_examples}
  \vspace{-14pt}
\end{figure}

Recent studies of collapse share a premise that collapse is caused by drift between old and new policies, and the mismatch between training and inference engines \citep{ deng2025grpo, qi2026rethinking,  zhang2026beyond}. Proposed solutions include importance sampling and trust region clipping \citep{schulman2015trpo, qi2025defeating, zheng2025stabilizing, zheng2025group}. These approaches are effective in practice, but do not predictably improve stability. We observe that removing more divergent tokens does not always lead to improved stability, and can cause collapse~(see \Cref{fig:dapo_smol_analysis}). 

We take an alternate approach in this study and do not assume all updates to divergent tokens are harmful. Instead of asking ``How far are we from the old policy?'', we ask ``How does a gradient update alter the local distribution?''. Our analysis depends on the probability  distribution of the current policy model $\pi_{\theta}$ and does not require explicitly modeling drift from $\pi_{\theta_\text{old}}$. Our approach provides an alternative lens for understanding instability when off-policy updates are unavoidable.

In \Cref{sec:gradient-token-taxonomy} we study the first-order gradient effect of policy updates on probability and entropy. We propose a taxonomy that theoretically predicts the separation between prior-aligned tokens that reduce entropy, and prior-opposed tokens which increase entropy. Exploration-like updates increase entropy and lead to collapsed generations, while prior-aligning tokens in negative rollouts also causes collapse. While some negative-advantage tokens may contain useful learning signal, selecting those tokens under coarse rewards is difficult and lies outside the scope of the current taxonomy. 

Motivated by these observations, we propose \emph{Winner Advantage Policy Optimization} (\method), a minimal modification to GRPO-style training. \method\ utilizes only positive-advantage completions in its policy-gradient update. Prompts with at least one winner are normalized over the group, while prompts with no positive-advantage completion contribute no policy-gradient. Unlike rejection fine-tuning or simple filtering, \method\ remains an online clipped policy-gradient method: it uses policy ratios and group-normalized advantages, but masks policy-gradient terms from non-winning completions.

We analyse \method\ in an idealized binary-reward setting. For a given prompt $x$, we show that various choices of normalization over positive rollouts ascend the same success-probability $q_x$, but with an adaptive factor that depends on $q_x$. For \method, this factor is $1-q_x$, which reduces update strength as prompts become saturated while preserving reinforcement from sampled successes.

Empirically, we evaluate \method\ comprehensively on \numina{}~\cite{wang2025kimina}, \datamath{}~\citep{lightman2024let}, \hotpot{}~\citep{ho2020constructing}, and \ott{}~\citep{chen2020open}, covering mathematical reasoning and multi-hop QA datasets. We compare across model families including \qwen{}~\citep{yang2025qwen3}, \smol{}~\citep{bakouch2025smollm3}, and \gemma{}~\citep{gemmateam2025gemma3technicalreport}. Across these settings, \method\ matches or improves over \GRPO{}, \DAPO{}~\citep{yu2025dapo}, and \GSPO{}~\citep{zheng2025group} in final accuracy and pass@$k$. The gains are clearest when baselines collapse; when baselines are stable, \method{} performs on par despite using only positive-advantage completions for policy-gradient updates. We also show out-of-distribution generalization of \method{} by evaluating \hotpot{} and \numina{} checkpoints on \wikidata{}~\citep{ho2020constructing} and \aime{}~\citep{dekoninck2026matharena} respectively.

\pgfplotsset{compat=1.18}
\definecolor{cottsmolsft50256648drdapo1001}{RGB}{128,177,211}
\definecolor{cottsmolbsft50256648drdapo1008}{RGB}{251,128,114}

\pgfplotstableread[col sep=comma]{plots/data/ott-smol-dr_dapo-10-0.1_f1_reward.csv}\rewardOne
\pgfplotstableread[col sep=comma]{plots/data/ott-smol-dr_dapo-10-0.8_f1_reward.csv}\rewardTwo
\pgfplotstableread[col sep=comma]{plots/data/ott-smol-dr_dapo-10-0.1_entropy.csv}\entropyOne
\pgfplotstableread[col sep=comma]{plots/data/ott-smol-dr_dapo-10-0.8_entropy.csv}\entropyTwo

\begin{figure*}[htbp] 
\centering

\hspace{-1.0cm}
\begin{minipage}[c]{0.66\textwidth}
\centering
\begin{tikzpicture}
    \begin{groupplot}[
        group style={
            group size=2 by 1,          
            horizontal sep=1.2cm,       
        },
        width=0.47\linewidth,           
        height=0.43\linewidth,          
        grid=both,
        grid style={line width=.1pt, draw=gray!25},
        major grid style={line width=.2pt, draw=gray!35},
        label style={font=\small},
        tick label style={font=\footnotesize},
        title style={font=\small\bfseries, yshift=-2pt},
        filter discard warning=false,
    ]

        \nextgroupplot[
            ylabel={F1 Reward},
            xlabel={Steps},
            x label style={yshift=0.20cm},
            title={(a) Reward},
            legend style={
                font=\scriptsize,
                draw=none, 
                fill=white,             
                fill opacity=0.6,
                text opacity=1,
                at={(0.01, 0.37)},
                anchor=north west, 
                legend columns=1,       
                row sep=0.1cm
            }
        ]
        
        \addplot[color=cottsmolsft50256648drdapo1001, line width=1.1pt, no marks, smooth, each nth point=5] 
            table [x index=0, y index=1] {\rewardOne};
        \addlegendentry{$\epsilon_{\text{low}}=0.9$}
        
        \addplot[color=cottsmolbsft50256648drdapo1008, line width=1.1pt, no marks, smooth, each nth point=5] 
            table [x index=0, y index=1] {\rewardTwo};
        \addlegendentry{$\epsilon_{\text{low}}=0.2$}
        
        \nextgroupplot[
            ylabel={Entropy},
            xlabel={Steps},
            x label style={yshift=0.20cm},
            title={(b) Entropy},
            legend style={
                font=\scriptsize,
                draw=none, 
                fill=white,             
                fill opacity=0.6,
                text opacity=1,
                at={(0.95, 0.45)},       
                anchor=east,            
                legend columns=1,       
                row sep=0.1cm
            }
        ]
        
        \addplot[color=cottsmolbsft50256648drdapo1008, line width=1.1pt, no marks, smooth, each nth point=5] 
            table [x index=0, y index=1] {\entropyTwo};
        \addlegendentry{$\epsilon_{\text{low}}=0.2$}
        
        \addplot[color=cottsmolsft50256648drdapo1001, line width=1.1pt, no marks, smooth, each nth point=5] 
            table [x index=0, y index=1] {\entropyOne};
        \addlegendentry{$\epsilon_{\text{low}}=0.9$}

    \end{groupplot}
\end{tikzpicture}
\end{minipage}%
\hspace{-0.50cm}%
\begin{minipage}[c]{0.30\textwidth}
\centering
\begin{tikzpicture}
    \begin{axis}[
        width=1.2\linewidth,
        height=\linewidth,         
        title={(c) Probabilities},
        xlabel={Sampled Probability},  
        ylabel={IR},
        label style={font=\small},
        title style={font=\small\bfseries, yshift=-2pt},
        x label style={xshift=0.18cm, yshift=0.35cm},
        y label style={yshift=-0.4cm},
        axis line style={draw=none},     
        tick style={draw=none},          
        xticklabels={,,},                
        yticklabels={,,},                
        xmin=0, xmax=1, ymin=0, ymax=1 
    ]
    \addplot graphics [xmin=0, xmax=1, ymin=0, ymax=1] {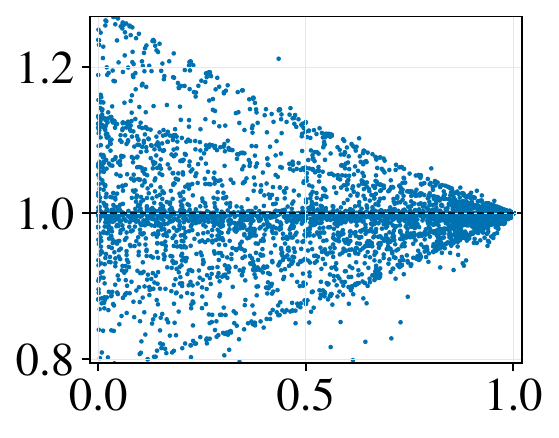};
    \end{axis}
\end{tikzpicture}
\end{minipage}
\vspace{-10pt}
\caption{We plot F1 reward (\subref{fig:dapo_smol_reward}) and entropy (\subref{fig:dapo_smol_entropy}) of \smol{} trained with DAPO under two negative clipping thresholds $\epsilon_\text{low}$ on \ott{}. More aggressive clipping (red) suppresses divergent negative tokens, but does not improve training stability, nor reduce entropy. Panel (\subref{fig:dapo_smol_probabilities}) shows importance ratio (IR) vs.\ sampled probability at one of the steps. The spread of importance ratio at lower probability indicates that clipping adversely affects lower-probability tokens \citep{qi2026rethinking}. In \Cref{sec:gradient-token-taxonomy} we show that restricting low-probability negative tokens leads to an increase in entropy and drives collapse.}
\label{fig:dapo_smol_analysis}

\phantomsubcaption\label{fig:dapo_smol_reward}
\phantomsubcaption\label{fig:dapo_smol_entropy}
\phantomsubcaption\label{fig:dapo_smol_probabilities}
\vspace{-10pt}
\end{figure*}
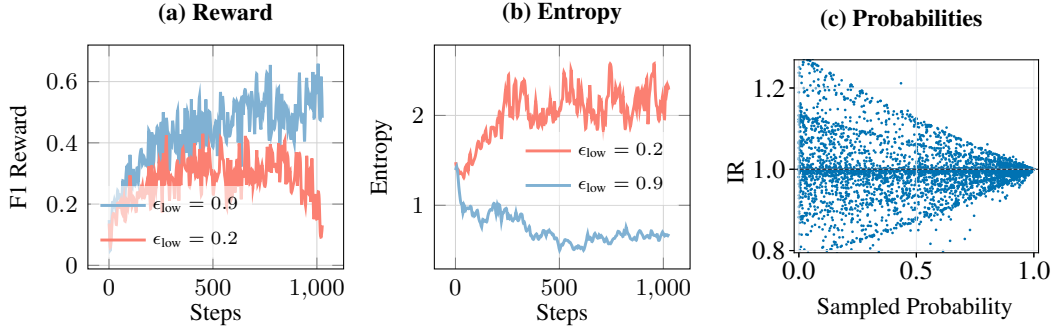

\section{Background}

RLVR commonly uses \GRPO-style objectives, where a rollout policy $\pi_{\theta_{\text{old}}}$ samples grouped completions, rewards are converted into group-relative advantages, and sequence-level advantages are applied through clipped importance-ratio updates, optionally with KL regularization to $\pi_{\theta_{\text{old}}}$ \citep{shao2024deepseekmath}. Variants such as \DAPO~\citep{yu2025dapo}, which modifies length normalization and clipping, and Dr.~GRPO~\citep{liuunderstanding}, which analyses reward and length-normalization biases, refine this framework, yet off-policy optimization under coarse rewards remains prone to collapse~\citep{deng2025grpo, qi2026rethinking}.

\paragraph{Policy mismatch and trust-regions.}

Instability is most often attributed to mismatch between $\pi_{\theta_{\text{old}}}$ and the current policy $\pi_{\theta}$. One line of work refines trust-region mechanisms: \GSPO{}~\citep{zheng2025group} addresses the use of token-level importance sampling with sequence-level clipping, while \citet{qi2026rethinking} show that ratio-based clipping can amplify harmful gradients on low-probability tokens. Another line targets train-inference mismatch through engineering fixes, including FP16 precision alignment~\citep{qi2025defeating}, adaptive learning-rate scheduling~\citep{zhang2026beyond}, and router synchronization for MoE architectures~\citep{ma2025stabilizing, ye2026adaptive}. These approaches all treat policy or system mismatch as central; our results instead motivate analysing collapse through token-level gradient effects.

\paragraph{Gradient analysis of training dynamics.}

Gradient-based analyses provide a complementary view. \citet{deng2025grpo} show that failed low-likelihood trajectories can induce likelihood displacement and drive collapse. In DPO-style \citep{rafailov2023direct} off-policy preference optimization, \citet{renlearning} identify a related ``squeezing effect'': negative updates reduce the likelihood of all responses, not only penalized ones, yielding a degenerate peaked distribution.

These works motivate our focus on how advantage sign interacts with the current token distribution during \GRPO{}-style training. Our taxonomy in Section~\ref{sec:gradient-token-taxonomy} unifies positive- and negative-advantage updates under this token-gradient view and motivates the positive-only update used by \method.

\section{A Gradient-Motivated Token Taxonomy}
\label{sec:gradient-token-taxonomy}
We train \smol{} on \ott{} with DAPO under different clipping thresholds (see Figure~\ref{fig:dapo_smol_analysis}). Lowering the clipping threshold $\epsilon_\text{low}$ removes divergent updates, but training collapses. Thus, divergence from $\pi_{\theta_{\text{old}}}$ alone does not predict stability. Figure \ref{fig:dapo_smol_probabilities} shows ratio-based clipping adversely affects low probability tokens since the spread of importance ratio is higher at lower probabilities. See Appendix \ref{app:additional_clip_eps_exps} for additional examples of collapse with GRPO, and different LLMs.

We now analyse training collapse through the local effect of per-token gradient updates and show that negative-advantage low probability tokens reduce entropy. We write the next-token distribution as $p=\operatorname{softmax}(z)$ for logits $z\in\mathbb{R}^V$ over vocabulary $[V]$. Let $A\in\mathbb{R}$ be the sequence-level advantage and consider the advantage-weighted negative log-likelihood, $\ell_s(z)=-A\log p_s$, where $s\in[V]$ is the sampled token. For any non-sampled token $i\neq s$, a small gradient descent step of size $\eta>0$ yields
\begin{equation}
\label{eq:delta_prob}
\Delta p_i
=
\eta A p_i\left(C(p)-p_s-p_i\right)
+O(\eta^2),
\end{equation}
\noindent where $C(p)=\sum_{j=1}^{V}p_j^2$ (proof in Appendix~\ref{app:gradient-token-taxonomy}). We note that in most practical settings the learning rate $\eta$ is a small number, so this first-order approximation is applicable.  Equation~\ref{eq:delta_prob} shows that a gradient update can increase the probability of non-sampled tokens. For negative-advantage samples, this is expected: suppressing the sampled token redistributes probability mass elsewhere. More surprisingly, it can also occur for positive-advantage samples. When $A>0$, a non-sampled token increases in probability if $p_s+p_i<C(p)$ and decreases if $p_s+p_i>C(p)$; for $A<0$, the inequalities reverse. Thus, reinforcing a low-probability successful token need not sharpen the distribution around that continuation, but can instead raise other low-probability alternatives and broaden the local distribution.

\pgfplotsset{compat=1.18}

\definecolor{barPeak}{RGB}{31, 119, 180} 
\definecolor{barValley}{RGB}{255, 127, 14} 

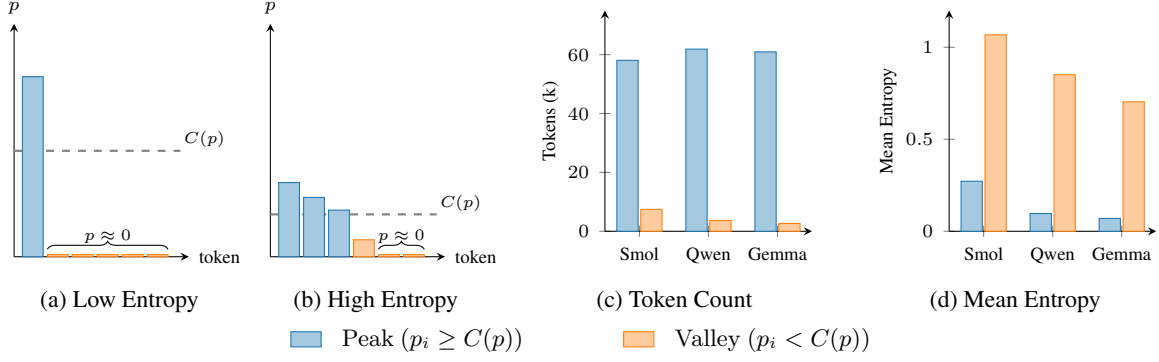
\begin{figure*}[t]
\centering
    \begin{subfigure}[b]{0.20\textwidth}
    \centering
    \begin{tikzpicture}[x=0.55cm,y=2.8cm] 
        \draw[->, >=stealth] (0,0) -- (4.2,0) node[right, font=\scriptsize] {token};
        \draw[->, >=stealth] (0,0) -- (0,1.1) node[above, font=\scriptsize] {\(p\)};
        \draw[dashed, draw=gray, thick] (0,0.5) -- (4.0,0.5) node[above right, black, inner sep=1pt, font=\tiny] {\(C(p)\)};
        \draw[fill=barPeak!40, draw=barPeak] (0.2,0) rectangle +(0.5, 0.85); 
        \draw[fill=barValley!40, draw=barValley] (0.8,0) rectangle +(0.5, 0.01);
        \draw[fill=barValley!40, draw=barValley] (1.4,0) rectangle +(0.5, 0.01);
        \draw[fill=barValley!40, draw=barValley] (2.0,0) rectangle +(0.5, 0.01);
        \draw[fill=barValley!40, draw=barValley] (2.6,0) rectangle +(0.5, 0.01);
        \draw[fill=barValley!40, draw=barValley] (3.2,0) rectangle +(0.5, 0.01);
        \draw[decorate, decoration={brace, amplitude=3pt, raise=2pt}] 
            (0.8, 0.01) -- (3.7, 0.01) 
            node[midway, above=4pt, font=\tiny, inner sep=0pt] {\( p \approx 0\)};
    \end{tikzpicture}
    \caption{Low Entropy}
    \label{fig:low-entropy}
    \end{subfigure}
    \hspace{-0.1pt}
    \begin{subfigure}[b]{0.20\textwidth}
    \centering
    \begin{tikzpicture}[x=0.55cm,y=2.8cm]
        \draw[->, >=stealth] (0,0) -- (4.2,0) node[right, font=\scriptsize] {token};
        \draw[->, >=stealth] (0,0) -- (0,1.1) node[above, font=\scriptsize] {\(p\)};
        \draw[dashed, draw=gray, thick] (0,0.2) -- (4.0,0.2) node[above right, black, inner sep=1pt, font=\tiny] {\(C(p)\)};
        \draw[fill=barPeak!40, draw=barPeak] (0.2,0) rectangle +(0.5, 0.35);
        \draw[fill=barPeak!40, draw=barPeak] (0.8,0) rectangle +(0.5, 0.28);
        \draw[fill=barPeak!40, draw=barPeak] (1.4,0) rectangle +(0.5, 0.22); 
        \draw[fill=barValley!40, draw=barValley] (2.0,0) rectangle +(0.5, 0.08); 
        \draw[fill=barValley!40, draw=barValley] (2.6,0) rectangle +(0.5, 0.01);
        \draw[fill=barValley!40, draw=barValley] (3.2,0) rectangle +(0.5, 0.01);
        \draw[decorate, decoration={brace, amplitude=3pt, raise=2pt}] 
            (2.6, 0.01) -- (3.7, 0.01) 
            node[midway, above=4pt, font=\tiny, inner sep=0pt] {\( p \approx 0\)};
    \end{tikzpicture}
    \caption{High Entropy}
    \label{fig:high-entropy}
    \end{subfigure}
    \hspace{-0.1pt}
    \begin{subfigure}[b]{0.27\textwidth} 
    \centering
    \begin{tikzpicture}
        \begin{axis} [
            ybar=1pt,
            width=\linewidth,
            height=4.5cm,          
            bar width=8pt,
            enlarge x limits=0.25,          
            ylabel={Tokens (k)},
            ylabel style={font=\scriptsize, yshift=-4pt},
            tick label style={font=\scriptsize},
            ymin=0, ymax=75, 
            xmin=0.5, xmax=3.5,
            xtick={1, 2, 3},
            xticklabels={Smol, Qwen, Gemma},
            grid=none, 
            axis y line=left,
            axis x line=bottom,
        ]   
            \addplot[draw=barPeak, fill=barPeak!40] coordinates {
                (1, 58.063)
                (2, 61.865)
                (3, 60.935)
            };
            \addplot[draw=barValley, fill=barValley!40] coordinates {
                (1, 7.438)
                (2, 3.671)
                (3, 2.654)
            };
        \end{axis}
    \end{tikzpicture}
    \caption{Token Count}
    \label{fig:token-count}
    \end{subfigure}
    \hspace{-0.1pt}
    \begin{subfigure}[b]{0.27\textwidth}
    \centering
    \begin{tikzpicture}
        \begin{axis} [
            ybar=1pt,
            width=\linewidth,
            height=4.5cm,          
            bar width=8pt,
            enlarge x limits=0.25,          
            ylabel={Mean Entropy},
            ylabel style={font=\scriptsize, yshift=-4pt},
            tick label style={font=\scriptsize},
            ymin=0, ymax=1.2, 
            xmin=0.5, xmax=3.5,
            xtick={1, 2, 3},
            xticklabels={Smol, Qwen, Gemma},
            grid=none, 
            axis y line=left,
            axis x line=bottom,
        ]   
            \addplot[draw=barPeak, fill=barPeak!40] coordinates {
                (1, 0.2719)
                (2, 0.0966)
                (3, 0.0698)
            }; 
            \addplot[draw=barValley, fill=barValley!40] coordinates {
                (1, 1.0676)
                (2, 0.8508)
                (3, 0.7035)
            };
        \end{axis}
    \end{tikzpicture}
    \caption{Mean Entropy}
    \label{fig:mean-entropy}
    \end{subfigure}
    \vspace{0.3cm}
    \begin{tikzpicture}
        \fill[barPeak!40, draw=barPeak] (0,0) rectangle (0.4,0.2) node[anchor=west, black, font=\small, xshift=4pt, yshift=-3pt] {\(\Peak\) \(\left(p_i \ge C(p)\right)\)};
        \fill[barValley!40, draw=barValley] (4.4,0) rectangle (4.8,0.2) node[anchor=west, black, font=\small, xshift=4pt, yshift=-3pt] {\(\Valley\) \(\left(p_i < C(p)\right)\)};
    \end{tikzpicture}
\vspace{-5pt}
\caption{\textbf{(\subref{fig:low-entropy}, \subref{fig:high-entropy}):} \Cref{eq:peak_valley_def} captures peaks and valleys in both low- and high-entropy distributions. \textbf{(\subref{fig:token-count}, \subref{fig:mean-entropy})} We report the number of sampled tokens and their average entropy for \smol{}, \qwen{}, and \gemma{} on NuminaMath-LEAN. Across models, valley tokens are sampled much less frequently than peak tokens, but have substantially higher entropy.}
\label{fig:peak_valley_unified_system}
\vspace{-10pt}
\end{figure*}

The quantity \(C(p)\) therefore acts as an adaptive reference level. Motivated by this, we define \(\peak\) and \(\valley\) sampled tokens as follows:
\begin{equation}
    \mathrm{Peak:} \,\, p_s \ge C(p), \ \ \mathrm{Valley:} \,\, p_s < C(p). \label{eq:peak_valley_def}
\end{equation}
Splitting these two cases by the sign of the advantage gives four regimes: $\Pos$-\(\peak\), $\Pos$-\(\valley\), $\Neg$-\(\peak\) and $\Neg$-\(\valley\).

This token taxonomy is more specific than entropy alone. Entropy is a property of the full next-token distribution and does not depend on which token was sampled. By contrast, the \(\peak/\valley\) distinction is conditioned on the token that is actually sampled. It therefore connects the sampled rollout to the local geometry of the softmax update. Figure~\ref{fig:peak_valley_unified_system}\subref{fig:low-entropy},\subref{fig:high-entropy} illustrates the  \(\peak/\valley\) categorization under two different probability profiles. Unlike a fixed probability cutoff, \(C(p)\) adapts to distribution sharpness: in flatter distributions, many tokens can be peak-like, while in concentrated distributions, only dominant tokens exceed \(C(p)\).

The threshold \(C(p)\) also satisfies the expected order properties. Let $p_\text{max}$, $p_\text{min}$ be the tokens with maximum and minimum probability under $p$, then
\begin{itemize}[leftmargin=*]
    \item \(p_{\max}\ge C(p)\). Hence every token attaining \(p_{\max}\) is a
    \(\peak\).
    \item \(p_{\min}\le C(p)\), with equality only when \(p\) is uniform on its
    support. Hence nearly every minimum-probability token is a \(\valley\).
\end{itemize}
\noindent Thus the taxonomy agrees with the intended intuition: high-probability sampled
tokens are \(\peak\)-like, while low-probability sampled tokens are \(\valley\)-like. The
advantage sign then determines whether the update reinforces or suppresses that
region of the distribution.

A similar first-order view also explains the entropy trends induced by the four
token types. Let \(H(p)=-\sum_i p_i\log p_i\). The change in entropy can be approximated as
\begin{align}\label{eq:delta_entropy}
\vspace{-2pt}
    \Delta H
    =
    -\eta A
    \Big[
        \Big(p_s\log p_s-\sum_i p_i^2\log p_i\Big) \nonumber \\
        + H(p)\Big(p_s-C(p)\Big) \Big] +O(\eta^2).
\vspace{-4pt}
\end{align}
For \(\valley\) tokens, the bracketed factor is negative (Appendix~\ref{app:entropy-direction}), so $\Pos$-\(\valley\) updates increase entropy and $\Neg$-\(\valley\) updates decrease it. The first term inside the bracket is zero in expectation, while for \(\peak\) tokens the second term is positive per definition in \Cref{eq:peak_valley_def}. Thus, $\Pos$-\(\peak\) updates usually decrease entropy, while $\Neg$-\(\peak\) updates tend to increase it. Appendix~\ref{app:entropy-direction} gives the full theoretical discussion; the next section verifies these predictions empirically during RL fine-tuning.

\section{Diagnosing Token-Level Instability in RLVR}
\label{sec:token-instability}

\paragraph{Isolating the four token regimes.}
We train \smol{} on \numina{} while masking the loss to each \(\peak\)--\(\valley\) and advantage-sign regime. Figure~\ref{fig:ablation_matrix}\subref{fig:ablation_reward},\subref{fig:ablation_entropy} shows distinct reward--entropy dynamics in each case. $\Pos$-\(\peak\) is stable but quickly plateaus, reinforcing already likely continuations, reducing entropy but offering limited pressure to explore alternative solution paths. $\Neg$-\(\peak\) and $\Pos$-\(\valley\), the two entropy-increasing regimes, lead to collapse: both move mass toward non-sampled alternatives, causing generations to drift toward increasingly random continuations and rewards to deteriorate.

\pgfplotsset{compat=1.18}

\definecolor{negValley}{RGB}{255, 150, 150}
\definecolor{negPeak}{RGB}{255, 100, 100} 
\definecolor{posValley}{RGB}{140, 220, 140}
\definecolor{posPeak}{RGB}{70, 180, 70} 
\definecolor{runChallenger}{RGB}{0, 168, 181} 
\definecolor{runFavorite}{RGB}{100, 110, 120}
\definecolor{runPos}{RGB}{44, 160, 44} 
\definecolor{runNeg}{RGB}{210, 115, 220}
\definecolor{runDrdapo}{RGB}{219, 0, 117}

\pgfplotstableread[col sep=comma]{plots/data/numina_smol_neg_challenger.csv}\dataNegChal
\pgfplotstableread[col sep=comma]{plots/data/numina_smol_neg_favorite.csv}\dataNegFav
\pgfplotstableread[col sep=comma]{plots/data/numina_smol_pos_challenger.csv}\dataPosChal
\pgfplotstableread[col sep=comma]{plots/data/numina_smol_pos_favorite.csv}\dataPosFav

\pgfplotstableread[col sep=comma]{plots/data/numina_smol_challenger.csv}\runChal
\pgfplotstableread[col sep=comma]{plots/data/numina_smol_favorite.csv}\runFav
\pgfplotstableread[col sep=comma]{plots/data/numina_smol_pos.csv}\runPosData
\pgfplotstableread[col sep=comma]{plots/data/numina_smol_neg.csv}\runNegData
\pgfplotstableread[col sep=comma]{plots/data/numina_smol_dr_dapo-10-0.1.csv}\runDrDapo

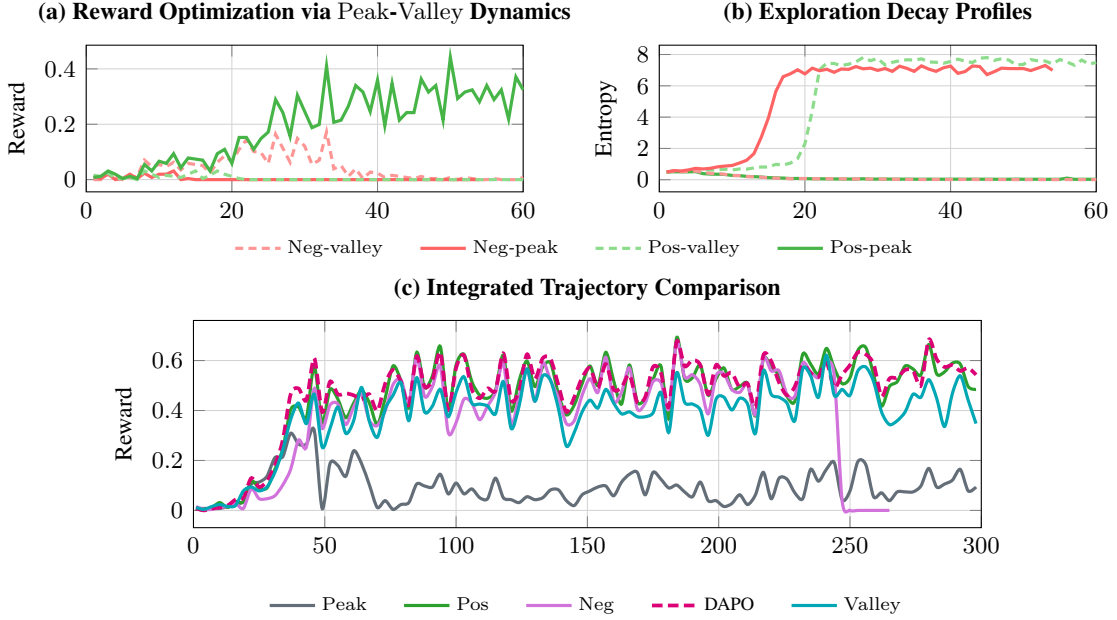
\begin{figure*}[htbp]
\centering
\begin{tikzpicture}

    \begin{groupplot}[
        group style={
            group size=2 by 2,          
            horizontal sep=1.8cm,
            vertical sep=1.7cm,
        },
        width=0.46\linewidth,
        height=0.22\linewidth,
        grid=both,
        grid style={line width=.1pt, draw=gray!25},
        major grid style={line width=.2pt, draw=gray!35},
        label style={font=\small},
        tick label style={font=\footnotesize},
        title style={font=\small\bfseries, yshift=-2pt},
        filter discard warning=false
    ]

        \nextgroupplot[
            restrict x to domain=0:60,
            xmin=0, xmax=60,
            ylabel={Reward},
            title={(a) Reward Optimization via \(\Peak\)-\(\Valley\) Dynamics}, 
            legend style={
                font=\scriptsize, draw=none, fill=none,
                at={(1.12, -0.25)},      
                anchor=north,
                legend columns=4,          
                /tikz/every even column/.append style={column sep=0.4cm} 
            }
        ]
        \addplot[color=negValley, line width=1.2pt, no marks, dashed, dash pattern=on 3pt off 2pt] table [x index=0, y index=1] {\dataNegChal};
        \addlegendentry{$\Neg$-\(\valley\)}

        \addplot[color=negPeak, line width=1.2pt, no marks, solid] table [x index=0, y index=1] {\dataNegFav};
        \addlegendentry{$\Neg$-\(\peak\)}

        \addplot[color=posValley, line width=1.2pt, no marks, dashed, dash pattern=on 3pt off 2pt] table [x index=0, y index=1] {\dataPosChal};
        \addlegendentry{$\Pos$-\(\valley\)}

        \addplot[color=posPeak, line width=1.2pt, no marks, solid] table [x index=0, y index=1] {\dataPosFav};
        \addlegendentry{$\Pos$-\(\peak\)}

        \nextgroupplot[
            restrict x to domain=0:60,
            xmin=0, xmax=60,
            ylabel={Entropy},
            title={(b) Exploration Decay Profiles}   
        ]
        \addplot[color=posValley, line width=1.2pt, no marks, dashed, dash pattern=on 3pt off 2pt] table [x index=0, y index=2] {\dataPosChal};
        \addplot[color=posPeak, line width=1.2pt, no marks, solid] table [x index=0, y index=2] {\dataPosFav};
        \addplot[color=negValley, line width=1.2pt, no marks, dashed, dash pattern=on 3pt off 2pt] table [x index=0, y index=2] {\dataNegChal};
        \addplot[color=negPeak, line width=1.2pt, no marks, solid] table [x index=0, y index=2] {\dataNegFav};

        \nextgroupplot[hide axis]

        \nextgroupplot[
            restrict x to domain=0:300,
            xmin=0, xmax=300,
            width=0.75\linewidth,
            height=0.27\linewidth,
            xshift=-0.24\linewidth,       
            ylabel={Reward},
            title={(c) Integrated Trajectory Comparison},
            legend style={
                font=\scriptsize,
                draw=none,
                fill=none,
                at={(0.5, -0.28)},        
                anchor=north,
                legend columns=5,          
                /tikz/every even column/.append style={column sep=0.4cm} 
            }
        ]

        \addplot[color=runFavorite, line width=1.2pt, no marks, smooth, each nth point=3] table [x index=0, y index=1] {\runFav};
        \addlegendentry{\(\Peak\)}

        \addplot[color=runPos, line width=1.2pt, no marks, smooth, each nth point=3] table [x index=0, y index=1] {\runPosData};
        \addlegendentry{$\Pos$}

        \addplot[color=runNeg, line width=1.2pt, no marks, smooth, each nth point=3] table [x index=0, y index=1] {\runNegData};
        \addlegendentry{$\Neg$}

        \addplot[color=runDrdapo, line width=1.4pt, no marks, dashed, dash pattern=on 4pt off 2pt, smooth, each nth point=3] table [x index=0, y index=1] {\runDrDapo};
        \addlegendentry{DAPO}

        \addplot[color=runChallenger, line width=1.2pt, no marks, smooth, each nth point=3] table [x index=0, y index=1] {\runChal};
        \addlegendentry{\(\Valley\)}

    \end{groupplot}
\end{tikzpicture}
\vspace{-10pt}
\caption{
Single-group taxonomy ablations on \smol{} trained on \numina{}.
Panels (\subref{fig:ablation_reward},\subref{fig:ablation_entropy}) vary \(\peak\)/\(\valley\) bounds, showing reward convergence and structural entropy decay. 
Panel (\subref{fig:ablation_entropy}) validates \Cref{eq:delta_entropy}: $\Pos$-\(\valley\) and $\Neg$-\(\peak\) increase entropy, while $\Pos$-\(\peak\) and $\Neg$-\(\valley\) decrease it. 
Entropy-increasing groups collapse rapidly across hyperparameters. 
Panel (\subref{fig:ablation_trajectory}) plots running-average reward for five primary configurations.
}
\label{fig:ablation_matrix}

\phantomsubcaption\label{fig:ablation_reward}
\phantomsubcaption\label{fig:ablation_entropy}
\phantomsubcaption\label{fig:ablation_trajectory} 
\vspace{-10pt}
\end{figure*}

$\Neg$-\(\valley\) shows a third pattern. It often improves the reward early, but eventually collapses. A negative update on a low-probability token suppresses an unlikely continuation and redistributes mass to other tokens. This can help when the token is spurious, but in high-entropy contexts where valley tokens are common (see \Cref{fig:peak_valley_unified_system}\subref{fig:mean-entropy}), it can also prematurely concentrate mass. The resulting failure is therefore low-entropy: generations become confident, repetitive, or generic (see \Cref{fig:collapse_examples}).

Overall, the quadrant experiments (Figure~\ref{fig:ablation_matrix}\subref{fig:ablation_reward},\subref{fig:ablation_entropy}) match the entropy predictions of the taxonomy: $\Pos$-\(\peak\) is stable but conservative; $\Neg$-\(\peak\) and $\Pos$-\(\valley\) are entropy-increasing and prone to random collapse; and $\Neg$-\(\valley\) can help initially but may induce overconfident collapse. Thus, instability depends not only on how far the update moves from the rollout policy, but on which token-level gradients survive the clipping rule. We show similar trends on \qwen{} in Appendix \ref{app:additional-token-dynamics}.

\paragraph{Taxonomy-based masking.}
$\Pos$-\(\peak\) is stable but fails to reach optimal rewards. An effective policy optimization must utilize the unstable but exploratory quadrants, i.e. $\Neg$-\(\peak\) or $\Pos$-\(\valley\). Group dynamics with more quadrants depends on the prior distribution of a model for a given dataset. However we observe consistent trends along four axis-aligned choices: $\Pos$, $\Neg$, $\Peak$, and $\Valley$. Each combines one entropy-increasing and one entropy-decreasing regime, retaining both exploration and exploitation pressure. Figure~\ref{fig:ablation_matrix}\subref{fig:ablation_trajectory} shows the resulting training dynamics. 

$\Peak$ training collapses rapidly despite having the stable $\Pos$-\(\peak\) component, suggesting that $\Neg$-\(\peak\) dominates by suppressing high-probability sampled tokens and sharply redistributing mass to alternatives. $\Neg$ training learns quickly but later collapses, consistent with prior observations \citep{zhusurprising}. Since $\Neg$-\(\peak\) appears in both masks, it is a plausible source of severe instability. $\Valley$ training is more stable and resembles a high-entropy filtering objective \citep{wangbeyond}. This matches the higher-entropy contexts in which valley tokens occur (Figure~\ref{fig:peak_valley_unified_system}\subref{fig:mean-entropy}). However, $\Valley$ underperforms both $\Pos$ and lenient-clipping baselines, likely because such tokens are sampled less often (see Figure~\ref{fig:peak_valley_unified_system}\subref{fig:token-count}). $\Valley$ tokens are also sensitive to ratio-based clipping and are aggressively removed (see Figure~\ref{fig:dapo_smol_probabilities}).

Overall the advantage sign is the most stable coarse filter. $\Pos$ training performs on par with the \DAPO{} baseline, suggesting that rewarded rollouts already contain sufficient learning signal. Unlike clipping or divergence-based constraints, this intervention introduces no additional regularizer or trust-region threshold; it simply restricts which advantage signs are allowed to update the model. 

\section{Winner Advantage Policy Optimization}
\label{sec:wapo}

From the previous section, we see that positive advantage updates contain both
exploitation and rewarded exploration: $\Pos$-\(\peak\) reinforces high-probability
successful tokens, while $\Pos$-\(\valley\) reinforces successful low-probability
tokens. This motivates a simple intervention: remove non-positive
advantage terms from the policy-gradient update. We call the resulting method \emph{Winner Advantage Policy Optimization}
(\method{}). Unlike rejection fine-tuning, \method{} remains an online
GRPO-style policy-gradient method: it uses sampled rollouts, group-normalized
advantages, token-level importance ratios, and clipping.

\paragraph{Binary-reward intuition.}
We first motivate \method{} in an idealized binary-reward setting. For a
fixed prompt \(x\), let $q_x = \Pr_{y\sim\pi_\theta(\cdot|x)}[r(x,y)=1]$
denote the probability that the current policy produces a correct response.
Since \(r(x,y)\in\{0,1\}\), the expected-reward objective is exactly the success
probability:
\vspace{-2pt}
\begin{equation}
J_x(\theta)
=
\mathbb{E}_{y\sim\pi_\theta(\cdot\mid x)}[r(x,y)]
=
q_x .
\vspace{-2pt}
\end{equation}
Therefore the policy gradient of the binary-reward objective is \(\nabla q_x\).

Let \(y_i=(y_{i1},\ldots,y_{iT_i})\) be the \(i\)-th rollout with sequence length $T_i$, $p_{ij}=\pi_\theta(y_{ij}\mid x,y_{i,<j})$ and $r_{i}$ denote $r(x, y_i)$. With \(G\) sampled rollouts, we have (see Appendix~\ref{app:wapo_analysis})
\vspace{-2pt}
\begin{equation}
\nabla q_x
\approx
\frac{1}{G}
\sum_{i=1}^{G}
\sum_{j=1}^{T_i}
r_i\nabla\log p_{ij} .
\label{eq:wapo_success_gradient}
\vspace{-2pt}
\end{equation}
We now consider a $\pos$-only, group-centered advantage update. Let \(\bar r = \frac{1}{G}\sum_i r_{i}\). For successful rollouts, we can define the advantage as:
\vspace{-2pt}
\begin{equation}
A_{i}^+ = (r_i-\bar r)\mathbbm{1}[r_i=1]
      = (1-\bar r)r_i .
\vspace{-2pt}
\end{equation}
Thus, normalizing by the group size gives
\vspace{-2pt}
\begin{equation}
\frac{1}{G}
\sum_{i=1}^{G}
\sum_{j=1}^{T_i}
A_i^+ \nabla \log p_{ij}
\approx
(1-q_x)\nabla q_x ,
\label{eq:wapo_one_minus_q}
\vspace{-2pt}
\end{equation}
where we use \(\bar r\approx q_x\) (see Appendix \ref{app:wapo_analysis}).

This shows that the $\pos$-only group-normalized update ascends the same direction as the binary policy gradient, but with an adaptive factor \(1-q_x\).  A different $\pos$-only normalization using the number of successful samples \(G\bar r\) yields an adaptive factor of $\tfrac{1-q_x}{q_x}$. See Appendix \ref{app:wapo_analysis} for more details. Both these adaptive factors have an intuitive interpretation: they automatically weight prompts according to success under current policy, putting more weight on harder prompts while attenuating easier ones. They differ in how aggressively they emphasize hard prompts. $\tfrac{1-q_x}{q_x}$  grows rapidly when $q_x$ is small, which may overemphasize rare or noisy successes under coarse rewards. We therefore use the bounded $1-q_x$ factor as the default in \method{}.

\paragraph{WAPO objective.}
Generalizing the above analysis to the continuous reward setting, for each prompt \(x\) we sample a group of \(G\) completions
\(\{y_i\}_{i=1}^{G}\) from the rollout policy \(\pi_{\theta_{\mathrm{old}}}\).
Let \(A_{i}^{+} = \max(A_{i}, 0)\), where \(A_{i}\) is the group-normalized advantage of completion \(i\). For token
\(j\) in rollout \(i\), define the token-level importance ratio $\rho_{ij}(\theta)=\frac{\pi_\theta(y_{ij}\mid x,y_{i,<j})}{\pi_{\theta_{\mathrm{old}}}(y_{ij}\mid x,y_{i,<j})}$. We optimize

\vspace{-2pt}
\begin{align}
\label{eq:wapo_objective}
\mathcal{J}_{\mathrm{WAPO}}(\theta)
& \!=\!
\mathbb{E}_{x\sim\mathcal{D}}
\bigg[
\frac{1}{GT}
\sum_{i=1}^{G}
\sum_{j=1}^{T_i} A_i^+\cdot
\\[-0.2em]
&
\min(\rho_{ij}(\theta),1{+}\epsilon)
\;\Big|\; \exists\,i \text{ s.t. } A_i^+>0
\bigg].
\nonumber
\vspace{-2pt}
\end{align}
Note that the principled policy gradient in Equation~\ref{eq:wapo_success_gradient} does not normalize by the individual sequence length, $T_i$. We adapt this in Equation~\ref{eq:wapo_objective}, where we apply a uniform normalization using a configured maximum sequence length $T$. This approach is also supported by \citet{liuunderstanding}, who show that this mitigates inherent biases related to sequence length. 
Since \method{} keeps only \(A_{i}>0\) terms, the lower clipping bound in \Cref{eq:wapo_objective} is
inactive for the retained terms.

\paragraph{Relation to positive-only baselines.}
\method{} is related to positive-only methods such as PSR
\citep{zhusurprising} and RAFT++ \citep{xiong2025minimalist}, which also remove failed or non-winning completions. However, these methods are mainly formulated for binary winner-selection settings: PSR does not use a clipped group-relative policy-gradient objective, and RAFT++ uses per sequence normalization and does not use mean-centered group-relative advantages. This can give short successful completions disproportionate weight and bias training toward terse answer templates (Figure \ref{fig:numina_smol_sequence_norm_b}). In contrast, \method{} keeps the \GRPO-style advantage formulation and masks only by the sign of the advantage. Thus, in binary-reward settings it acts as a winner-only update, while in continuous-reward settings it naturally selects completions that score above the group baseline. Figure~\ref{fig:numina_smol_sequence_norm_a} compares PSR, RAFT++, and the two \method{} normalizations on \numina{} with \smol{} using exact match as the binary reward. Both \method{} variants outperform positive-only baselines.
\pgfplotsset{compat=1.18}

\definecolor{seqOneMinusQ}{RGB}{128, 177, 211}
\definecolor{seqOneMinusQOverQ}{RGB}{251, 128, 114}
\definecolor{seqRaft}{RGB}{190, 186, 218} 
\definecolor{seqPsr}{RGB}{253, 180, 98}

\pgfplotstableread[col sep=comma]{plots/data/numina_smol_mask_neg_grp_new.csv}\numinaSmolOneMinusQ
\pgfplotstableread[col sep=comma]{plots/data/numina_smol_mask_neg_seq_new.csv}\numinaSmolOneMinusQOverQ
\pgfplotstableread[col sep=comma]{plots/data/numina_smol_raftpp.csv}\numinaSmolRaft
\pgfplotstableread[col sep=comma]{plots/data/numina_smol_psr.csv}\numinaSmolPsr

\begin{figure}[t]
\centering
\hspace*{-0.3cm}
\begin{tikzpicture}
    \begin{axis}[
        name=rewardplot,
        width=0.98\linewidth,
        height=0.55\linewidth,
        grid=both,
        grid style={line width=.1pt, draw=gray!25},
        major grid style={line width=.2pt, draw=gray!35},
        xlabel={Steps},
        ylabel={Reward},
        label style={font=\small},
        tick label style={font=\footnotesize},
        title={(a) Positive-only ablations on \numina{}},
        title style={font=\small\bfseries, yshift=-2pt},
        legend style={
            font=\scriptsize,
            draw=none,
            fill=white,
            fill opacity=0.85,
            text opacity=1,
            at={(0.5, -0.32)},
            anchor=north,
            legend columns=4,
            /tikz/every even column/.append style={column sep=0.3cm}
        },
        filter discard warning=false
    ]
        \addplot[color=seqOneMinusQ, line width=1.1pt, no marks, smooth, each nth point=3]
            table [x=step, y=reward] {\numinaSmolOneMinusQ};
        \addlegendentry{$1-q_x$}

        \addplot[color=seqOneMinusQOverQ, line width=1.1pt, no marks, smooth, each nth point=3]
            table [x=step, y=reward] {\numinaSmolOneMinusQOverQ};
        \addlegendentry{$(1-q_x)/q_x$}

        \addplot[color=seqRaft, line width=1.1pt, no marks, dashed, dash pattern=on 3pt off 2pt, smooth, each nth point=3]
            table [x=step, y=reward] {\numinaSmolRaft};
        \addlegendentry{RAFT++}

        \addplot[color=seqPsr, line width=1.1pt, no marks, dashed, dash pattern=on 3pt off 2pt, smooth, each nth point=3]
            table [x=step, y=reward] {\numinaSmolPsr};
        \addlegendentry{PSR}
    \end{axis}
    \node[overlay, inner sep=0pt] at (rewardplot.north west) {%
        \phantomsubcaption\label{fig:numina_smol_sequence_norm_a}%
    };

    \node[
        anchor=north,
        yshift=-1.6cm,
        fill=red!5,
        draw=red!40,
        rounded corners=2pt,
        inner sep=4pt,
        text width=0.98\linewidth,
        align=left,
        font=\scriptsize
    ] at (rewardplot.south) {%
        \phantomsubcaption\label{fig:numina_smol_sequence_norm_b}%
        \textbf{(b) Sequence-level normalization short-answer bias.}
        RAFT++ and PSR normalize at the sequence level, giving each selected
        completion comparable mass regardless of its token length. In late
        RAFT++ rollouts, generations often collapse to terse answer templates:
        \vspace{1pt}

        {\ttfamily
        <think> Okay, the answer is 3. So Answer: 3. </think> Answer: 3
        }
    };
\end{tikzpicture}
\vspace{-12pt}
\caption{Comparison of \method{} with RAFT++ and PSR on NuminaMath-LEAN using SmolLM3-3B. RAFT++ exhibits short-answer bias, while PSR plateaus without importance ratios.}
\label{fig:numina_smol_sequence_norm}
\vspace{-8pt}
\end{figure}

\section{Results}
\definecolor{vfBlue}{HTML}{0072B2}
\definecolor{vfOrange}{HTML}{E69F00}
\definecolor{vfRed}{HTML}{D55E00}
\definecolor{vfSky}{HTML}{56B4E9}
\definecolor{vfGray}{HTML}{6B7280}

\tikzset{
wapo panel label/.style={anchor=south east, align=right, fill=white, fill opacity=0.70, text opacity=1, inner sep=0.45pt, font=\fontsize{4.2pt}{4.5pt}\selectfont},
}

\makeatletter
\pgfdeclareplotmark{wapoCircle}{%
  \pgfpathcircle{\pgfpointorigin}{\pgfplotmarksize}%
  \pgfusepathqfillstroke%
}
\pgfdeclareplotmark{wapoOpenCircle}{%
  \pgfpathcircle{\pgfpointorigin}{\pgfplotmarksize}%
  \pgfusepathqstroke%
}
\pgfdeclareplotmark{wapoSquare}{%
  \pgfpathrectangle{\pgfqpoint{-\pgfplotmarksize}{-\pgfplotmarksize}}{\pgfqpoint{2\pgfplotmarksize}{2\pgfplotmarksize}}%
  \pgfusepathqfillstroke%
}
\pgfdeclareplotmark{wapoDiamond}{%
  \pgfpathmoveto{\pgfqpoint{0pt}{\pgfplotmarksize}}%
  \pgfpathlineto{\pgfqpoint{\pgfplotmarksize}{0pt}}%
  \pgfpathlineto{\pgfqpoint{0pt}{-\pgfplotmarksize}}%
  \pgfpathlineto{\pgfqpoint{-\pgfplotmarksize}{0pt}}%
  \pgfpathclose%
  \pgfusepathqfillstroke%
}
\pgfdeclareplotmark{wapoTriangle}{%
  \pgfpathmoveto{\pgfqpoint{0pt}{\pgfplotmarksize}}%
  \pgfpathlineto{\pgfqpoint{.866\pgfplotmarksize}{-.5\pgfplotmarksize}}%
  \pgfpathlineto{\pgfqpoint{-.866\pgfplotmarksize}{-.5\pgfplotmarksize}}%
  \pgfpathclose%
  \pgfusepathqfillstroke%
}
\pgfdeclareplotmark{wapoCross}{%
  \pgfpathmoveto{\pgfqpoint{-.85\pgfplotmarksize}{-.85\pgfplotmarksize}}%
  \pgfpathlineto{\pgfqpoint{.85\pgfplotmarksize}{.85\pgfplotmarksize}}%
  \pgfpathmoveto{\pgfqpoint{-.85\pgfplotmarksize}{.85\pgfplotmarksize}}%
  \pgfpathlineto{\pgfqpoint{.85\pgfplotmarksize}{-.85\pgfplotmarksize}}%
  \pgfusepathqstroke%
}
\makeatother

\pgfplotsset{
wapo eval axis/.style={
width=0.332\linewidth,
height=0.205\linewidth,
xlabel={},
ylabel={Score},
xlabel style={font=\scriptsize, inner sep=0pt},
ylabel style={font=\scriptsize, inner sep=1pt},
tick label style={font=\scriptsize},
ymajorgrids=true,
xmajorgrids=true,
grid style={draw=vfGray!14, line width=0.2pt},
axis line style={vfGray!75},
tick style={vfGray!75},
line width=0.65pt,
mark size=1.0pt,
xtick distance=100,
scaled ticks=false,
every axis plot/.append style={line width=0.65pt},
},
wapo drdapo plot/.style={mark=wapoDiamond, color=vfSky},
wapo grpo plot/.style={mark=wapoCircle, color=vfBlue},
wapo gspo plot/.style={mark=wapoSquare, color=vfOrange},
wapo method plot/.style={mark=wapoTriangle, color=vfRed},
wapo base model plot/.style={mark=wapoOpenCircle, color=vfGray},
wapo unmask plot/.style={mark=wapoCross, color=vfRed},
wapo all plot/.style={mark=wapoTriangle, color=vfRed, dashed},
}

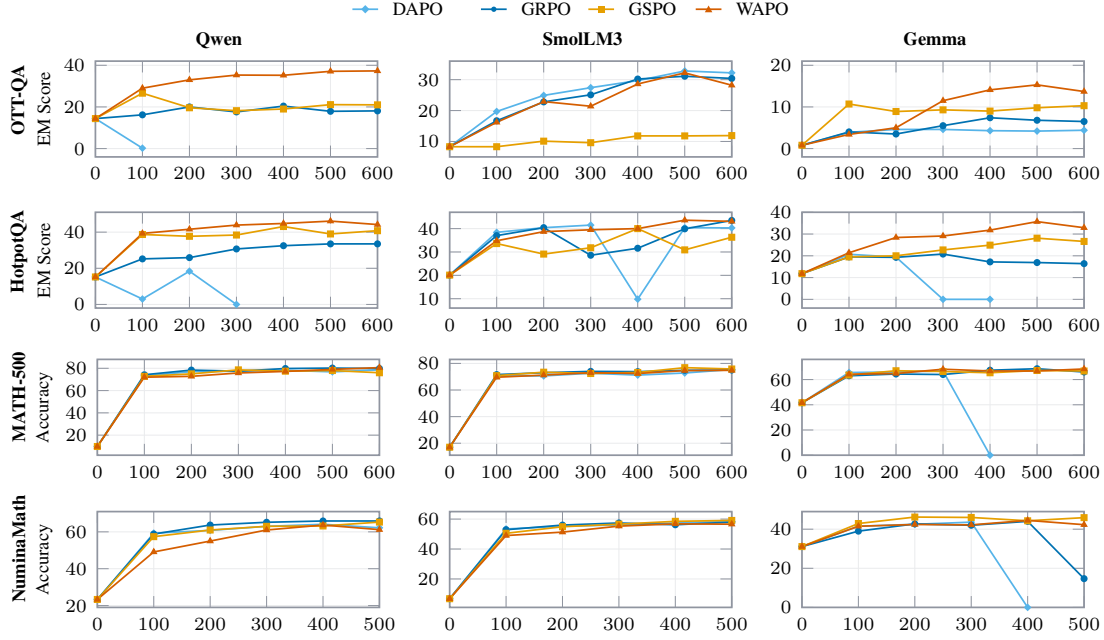
\begin{figure*}[t]
\centering
\setlength{\tabcolsep}{0.6pt}
\renewcommand{\arraystretch}{0.68}
\begin{tikzpicture}[baseline=-0.5ex]
\draw[color=vfSky, line width=0.65pt] (0.00,0) -- (0.34,0);
\fill[color=vfSky] (0.17,0.055) -- (0.225,0) -- (0.17,-0.055) -- (0.115,0) -- cycle;
\node[anchor=west, font=\scriptsize] at (0.40,0) {DAPO};
\draw[color=vfBlue, line width=0.65pt] (1.70,0) -- (2.04,0);
\fill[color=vfBlue] (1.87,0) circle[radius=0.035];
\node[anchor=west, font=\scriptsize] at (2.10,0) {GRPO};
\draw[color=vfOrange, line width=0.65pt] (3.12,0) -- (3.46,0);
\fill[color=vfOrange] (3.25,-0.04) rectangle (3.33,0.04);
\node[anchor=west, font=\scriptsize] at (3.52,0) {GSPO};
\draw[color=vfRed, line width=0.65pt] (4.54,0) -- (4.88,0);
\fill[color=vfRed] (4.71,0.05) -- (4.66,-0.04) -- (4.76,-0.04) -- cycle;
\node[anchor=west, font=\scriptsize] at (4.94,0) {\method};
\end{tikzpicture}
\begin{tabular}{ccc}
\scriptsize\textbf{Qwen} & \scriptsize\textbf{SmolLM3} & \scriptsize\textbf{Gemma} \\[-0.4mm]
\llap{\raisebox{0.07\linewidth}{\rotatebox[origin=c]{90}{\scriptsize\textbf{OTT-QA}}}\hspace{1mm}}%
\begin{tikzpicture}
\begin{axis}[wapo eval axis, height=0.178\linewidth, xmin=0, xmax=600, ymin=-4, ymax=42, ylabel={EM Score}]
\addplot[wapo drdapo plot] coordinates {(0,14.40) (100,0.20)};
\addplot[wapo grpo plot] coordinates {(0,14.40) (100,16.20) (200,20.00) (300,17.60) (400,20.40) (500,17.90) (600,18.10)};
\addplot[wapo gspo plot] coordinates {(0,14.40) (100,26.60) (200,19.60) (300,18.20) (400,19.00) (500,21.10) (600,21.00)};
\addplot[wapo method plot] coordinates {(0,14.40) (100,29.00) (200,33.00) (300,35.30) (400,35.20) (500,37.10) (600,37.30)};
\end{axis}
\end{tikzpicture} &
\begin{tikzpicture}
\begin{axis}[wapo eval axis, height=0.178\linewidth, xmin=0, xmax=600, ymin=5, ymax=36, ylabel={}]
\addplot[wapo drdapo plot] coordinates {(0,8.30) (100,19.70) (200,24.90) (300,27.40) (400,29.70) (500,32.80) (600,32.20)};
\addplot[wapo grpo plot] coordinates {(0,8.30) (100,16.70) (200,22.80) (300,25.10) (400,30.20) (500,31.10) (600,30.40)};
\addplot[wapo gspo plot] coordinates {(0,8.30) (100,8.30) (200,10.10) (300,9.60) (400,11.80) (500,11.80) (600,11.90)};
\addplot[wapo method plot] coordinates {(0,8.30) (100,16.20) (200,22.90) (300,21.40) (400,28.60) (500,32.20) (600,28.20)};
\end{axis}
\end{tikzpicture} &
\begin{tikzpicture}
\begin{axis}[wapo eval axis, height=0.178\linewidth, xmin=0, xmax=600, ymin=-2, ymax=21, ylabel={}]
\addplot[wapo drdapo plot] coordinates {(0,0.80) (100,3.90) (200,4.60) (300,4.60) (400,4.30) (500,4.20) (600,4.40)};
\addplot[wapo grpo plot] coordinates {(0,0.80) (100,4.0) (200,3.50) (300,5.50) (400,7.40) (500,6.80) (600,6.50)};
\addplot[wapo gspo plot] coordinates {(0,0.80) (100,10.70) (200,8.90) (300,9.30) (400,9.0) (500,9.80) (600,10.30)};
\addplot[wapo method plot] coordinates {(0,0.80) (100,3.40) (200,5.0) (300,11.50) (400,14.10) (500,15.30) (600,13.70)};
\end{axis}
\end{tikzpicture} \\
\llap{\raisebox{0.07\linewidth}{\rotatebox[origin=c]{90}{\scriptsize\textbf{HotpotQA}}}\hspace{1mm}}%
\begin{tikzpicture}
\begin{axis}[wapo eval axis, height=0.178\linewidth, xmin=0, xmax=600, ymin=-2, ymax=51, ylabel={EM Score}]
\addplot[wapo drdapo plot] coordinates {(0,15.20) (100,3.00) (200,18.40) (300, 0)};
\addplot[wapo grpo plot] coordinates {(0,15.20) (100,25.20) (200,25.90) (300,30.70) (400,32.50) (500,33.50) (600,33.50)};
\addplot[wapo gspo plot] coordinates {(0,15.20) (100,38.70) (200,37.70) (300,38.40) (400,43.10) (500,39.00) (600,40.80)};
\addplot[wapo method plot] coordinates {(0,15.20) (100,39.30) (200,41.60) (300,43.90) (400,44.80) (500,46.10) (600,44.20)};
\end{axis}
\end{tikzpicture} &
\begin{tikzpicture}
\begin{axis}[wapo eval axis, height=0.178\linewidth, xmin=0, xmax=600, ymin=6, ymax=47, ylabel={}]
\addplot[wapo drdapo plot] coordinates {(0,20.10) (100,38.40) (200,40.40) (300,41.50) (400,9.80) (500,40.40) (600,40.30)};
\addplot[wapo grpo plot] coordinates {(0,20.10) (100,37.00) (200,40.40) (300,28.60) (400,31.60) (500,39.90) (600,43.50)};
\addplot[wapo gspo plot] coordinates {(0,20.10) (100,33.60) (200,29.10) (300,31.80) (400,40.00) (500,30.90) (600,36.30)};
\addplot[wapo method plot] coordinates {(0,20.10) (100,34.80) (200,38.70) (300,39.50) (400,40.00) (500,43.60) (600,43.10)};
\end{axis}
\end{tikzpicture} &
\begin{tikzpicture}
\begin{axis}[wapo eval axis, height=0.178\linewidth, xmin=0, xmax=600, ymin=-4, ymax=40, ylabel={}]
\addplot[wapo drdapo plot] coordinates {(0,11.80) (100,20.70) (200,19.40) (300,0.00) (400,0.00)};
\addplot[wapo grpo plot] coordinates {(0,11.80) (100,19.50) (200,19.30) (300,20.80) (400,17.20) (500,16.90) (600,16.40)};
\addplot[wapo gspo plot] coordinates {(0,11.80) (100,19.50) (200,20.00) (300,22.70) (400,24.90) (500,28.10) (600,26.60)};
\addplot[wapo method plot] coordinates {(0,11.80) (100,21.40) (200,28.40) (300,29.10) (400,31.80) (500,35.70) (600,32.90)};
\end{axis}
\end{tikzpicture} \\
\llap{\raisebox{0.07\linewidth}{\rotatebox[origin=c]{90}{\scriptsize\textbf{MATH-500}}}\hspace{1mm}}%
\begin{tikzpicture}
\begin{axis}[wapo eval axis, height=0.178\linewidth, xmin=0, xmax=600, ymin=2, ymax=88, ylabel={Accuracy}]
\addplot[wapo drdapo plot] coordinates {(0,9.80) (100,73.20) (200,77.00) (300,77.40) (400,78.20) (500,76.80) (600,78.60)};
\addplot[wapo grpo plot] coordinates {(0,9.80) (100,74.20) (200,78.40) (300,77.40) (400,79.80) (500,80.20) (600,79.80)};
\addplot[wapo gspo plot] coordinates {(0,9.80) (100,72.80) (200,75.00) (300,78.60) (400,77.40) (500,78.00) (600,76.00)};
\addplot[wapo method plot] coordinates {(0,9.80) (100,72.00) (200,72.80) (300,75.80) (400,77.20) (500,79.00) (600,80.60)};
\end{axis}
\end{tikzpicture} &
\begin{tikzpicture}
\begin{axis}[wapo eval axis, height=0.178\linewidth, xmin=0, xmax=600, ymin=11, ymax=83, ylabel={}]
\addplot[wapo drdapo plot] coordinates {(0,17.00) (100,70.60) (200,70.60) (300,72.60) (400,71.20) (500,72.80) (600,75.40)};
\addplot[wapo grpo plot] coordinates {(0,17.00) (100,71.60) (200,73.00) (300,74.00) (400,73.80) (500,74.80) (600,75.00)};
\addplot[wapo gspo plot] coordinates {(0,17.00) (100,70.80) (200,73.40) (300,72.40) (400,73.20) (500,76.80) (600,75.80)};
\addplot[wapo method plot] coordinates {(0,17.00) (100,69.60) (200,71.20) (300,73.00) (400,72.40) (500,74.60) (600,74.80)};
\end{axis}
\end{tikzpicture} &
\begin{tikzpicture}
\begin{axis}[wapo eval axis, height=0.178\linewidth, xmin=0, xmax=600, ymin=0, ymax=76, ylabel={}]
\addplot[wapo drdapo plot] coordinates {(0,41.60) (100,65.60) (200,66.00) (300,66.80) (400,0.00)};
\addplot[wapo grpo plot] coordinates {(0,41.60) (100,63.00) (200,64.40) (300,64.00) (400,67.40) (500,68.60) (600,66.00)};
\addplot[wapo gspo plot] coordinates {(0,41.60) (100,63.60) (200,67.00) (300,66.40) (400,65.40) (500,67.20) (600,66.80)};
\addplot[wapo method plot] coordinates {(0,41.60) (100,64.20) (200,64.80) (300,68.20) (400,66.80) (500,66.80) (600,68.40)};
\end{axis}
\end{tikzpicture} \\
\llap{\raisebox{0.07\linewidth}{\rotatebox[origin=c]{90}{\scriptsize\textbf{NuminaMath}}}\hspace{1mm}}%
\begin{tikzpicture}
\begin{axis}[wapo eval axis, height=0.178\linewidth, xmin=0, xmax=500, ymin=19, ymax=71, ylabel={Accuracy}]
\addplot[wapo drdapo plot] coordinates {(0,23.30) (100,59.30) (200,60.80) (300,63.00) (400,64.10) (500,62.20)};
\addplot[wapo grpo plot] coordinates {(0,23.30) (100,58.80) (200,63.70) (300,65.20) (400,65.90) (500,65.90)};
\addplot[wapo gspo plot] coordinates {(0,23.30) (100,57.40) (200,61.00) (300,63.00) (400,63.20) (500,65.30)};
\addplot[wapo method plot] coordinates {(0,23.30) (100,49.10) (200,55.0) (300,61.00) (400,63.70) (500,61.20)};
\end{axis}
\end{tikzpicture} &
\begin{tikzpicture}
\begin{axis}[wapo eval axis, height=0.178\linewidth, xmin=0, xmax=500, ymin=1, ymax=65, ylabel={}]
\addplot[wapo drdapo plot] coordinates {(0,6.80) (100,53.40) (200,55.20) (300,56.10) (400,57.40) (500,58.40)};
\addplot[wapo grpo plot] coordinates {(0,6.80) (100,52.90) (200,56.00) (300,57.40) (400,56.20) (500,58.00)};
\addplot[wapo gspo plot] coordinates {(0,6.80) (100,50.50) (200,54.90) (300,56.60) (400,58.50) (500,59.10)};
\addplot[wapo method plot] coordinates {(0,6.80) (100,49.00) (200,51.30) (300,55.30) (400,56.80) (500,56.60)};
\end{axis}
\end{tikzpicture} &
\begin{tikzpicture}
\begin{axis}[wapo eval axis, height=0.178\linewidth, xmin=0, xmax=500, ymin=0, ymax=49, ylabel={}]
\addplot[wapo drdapo plot] coordinates {(0,31.10) (100,41.50) (200,42.50) (300,43.60) (400,0)};
\addplot[wapo grpo plot] coordinates {(0,31.10) (100,39.00) (200,42.70) (300,42.00) (400,44.00) (500,14.70)};
\addplot[wapo gspo plot] coordinates {(0,31.10) (100,42.90) (200,46.20) (300,46.00) (400,44.30) (500,45.90)};
\addplot[wapo method plot] coordinates {(0,31.10) (100,41.50) (200,42.40) (300,42.20) (400,44.50) (500,42.30)};
\end{axis}
\end{tikzpicture} \\
\end{tabular}
\vspace{-6pt}
\caption{\textbf{Evaluation curves for exact match and accuracy.} Rows show datasets, columns show model families. Each subplot reports checkpoint performance up to 600 training steps, with step 0 initialized from the shared base model for math tasks or the SFT-coldstart model for multihop QA. Compared with DAPO, GRPO, and GSPO, \method{} is generally more stable and achieves stronger final performance across settings.}
\label{fig:eval_curves_em}
\end{figure*}
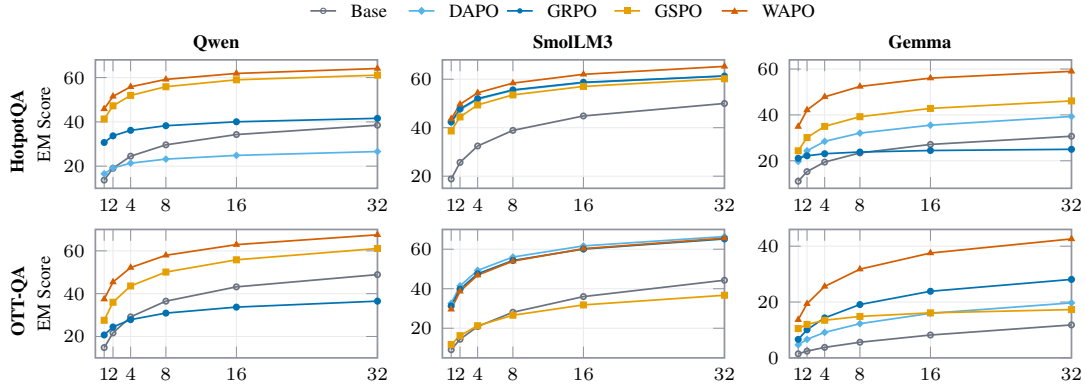
\begin{figure*}[!t]
\centering
\setlength{\tabcolsep}{0.6pt}
\renewcommand{\arraystretch}{0.74}
\begin{tikzpicture}[baseline=-0.5ex]
\draw[color=vfGray, line width=0.65pt] (0.00,0) -- (0.34,0);
\draw[color=vfGray, line width=0.45pt] (0.17,0) circle[radius=0.035];
\node[anchor=west, font=\scriptsize] at (0.40,0) {Base};
\draw[color=vfSky, line width=0.65pt] (1.30,0) -- (1.64,0);
\fill[color=vfSky] (1.47,0.055) -- (1.525,0) -- (1.47,-0.055) -- (1.415,0) -- cycle;
\node[anchor=west, font=\scriptsize] at (1.70,0) {DAPO};
\draw[color=vfBlue, line width=0.65pt] (2.60,0) -- (2.94,0);
\fill[color=vfBlue] (2.77,0) circle[radius=0.035];
\node[anchor=west, font=\scriptsize] at (3.00,0) {GRPO};
\draw[color=vfOrange, line width=0.65pt] (4.02,0) -- (4.36,0);
\fill[color=vfOrange] (4.15,-0.04) rectangle (4.23,0.04);
\node[anchor=west, font=\scriptsize] at (4.42,0) {GSPO};
\draw[color=vfRed, line width=0.65pt] (5.44,0) -- (5.78,0);
\fill[color=vfRed] (5.61,0.05) -- (5.56,-0.04) -- (5.66,-0.04) -- cycle;
\node[anchor=west, font=\scriptsize] at (5.84,0) {WAPO};
\end{tikzpicture}
\begin{tabular}{ccc}
\scriptsize\textbf{Qwen} & \scriptsize\textbf{SmolLM3} & \scriptsize\textbf{Gemma} \\[-0.4mm]
\llap{\raisebox{0.07\linewidth}{\rotatebox[origin=c]{90}{\scriptsize\textbf{HotpotQA}}}\hspace{1mm}}%
\begin{tikzpicture}
\begin{axis}[wapo eval axis, xmin=0, xmax=32, ymin=10, ymax=68, ylabel={EM Score}, xtick={1,2,4,8,16,32}]
\addplot[wapo base model plot] coordinates {(1,13.71) (2,19.03) (4,24.51) (8,29.62) (16,34.27) (32,38.50)};
\addplot[wapo drdapo plot] coordinates {(1,16.56) (2,19.18) (4,21.36) (8,23.18) (16,24.86) (32,26.60)};
\addplot[wapo grpo plot] coordinates {(1,30.69) (2,33.69) (4,36.18) (8,38.27) (16,40.04) (32,41.60)};
\addplot[wapo gspo plot] coordinates {(1,41.30) (2,47.25) (4,51.95) (8,55.88) (16,59.00) (32,61.10)};
\addplot[wapo method plot] coordinates {(1,45.84) (2,51.52) (4,55.83) (8,59.19) (16,61.87) (32,64.10)};
\end{axis}
\end{tikzpicture} &
\begin{tikzpicture}
\begin{axis}[wapo eval axis, xmin=0, xmax=32, ymin=15, ymax=68, ylabel={}, xtick={1,2,4,8,16,32}]
\addplot[wapo base model plot] coordinates {(1,18.84) (2,25.67) (4,32.45) (8,38.90) (16,44.85) (32,50.00)};
\addplot[wapo drdapo plot] coordinates {(1,42.01) (2,47.35) (4,51.72) (8,55.40) (16,58.55) (32,61.40)};
\addplot[wapo grpo plot] coordinates {(1,42.44) (2,47.80) (4,52.01) (8,55.60) (16,58.74) (32,61.30)};
\addplot[wapo gspo plot] coordinates {(1,38.66) (2,44.44) (4,49.39) (8,53.54) (16,57.03) (32,60.20)};
\addplot[wapo method plot] coordinates {(1,43.62) (2,49.65) (4,54.39) (8,58.40) (16,62.03) (32,65.30)};
\end{axis}
\end{tikzpicture} &
\begin{tikzpicture}
\begin{axis}[wapo eval axis, xmin=0, xmax=32, ymin=8, ymax=64, ylabel={}, xtick={1,2,4,8,16,32}]
\addplot[wapo base model plot] coordinates {(1,11.08) (2,15.27) (4,19.43) (8,23.42) (16,27.13) (32,30.70)};
\addplot[wapo drdapo plot] coordinates {(1,19.72) (2,24.35) (4,28.47) (8,32.07) (16,35.52) (32,39.30)};
\addplot[wapo grpo plot] coordinates {(1,21.09) (2,22.18) (4,23.08) (8,23.82) (16,24.48) (32,25.00)};
\addplot[wapo gspo plot] coordinates {(1,24.36) (2,30.12) (4,35.00) (8,39.22) (16,42.83) (32,46.10)};
\addplot[wapo method plot] coordinates {(1,34.91) (2,42.14) (4,47.87) (8,52.40) (16,56.05) (32,59.00)};
\end{axis}
\end{tikzpicture} \\
\llap{\raisebox{0.07\linewidth}{\rotatebox[origin=c]{90}{\scriptsize\textbf{OTT-QA}}}\hspace{1mm}}%
\begin{tikzpicture}
\begin{axis}[wapo eval axis, xmin=0, xmax=32, ymin=10, ymax=70, ylabel={EM Score}, xtick={1,2,4,8,16,32}]
\addplot[wapo base model plot] coordinates {(1,14.82) (2,21.66) (4,29.15) (8,36.50) (16,43.18) (32,48.90)};
\addplot[wapo grpo plot] coordinates {(1,20.70) (2,24.44) (4,27.87) (8,30.92) (16,33.72) (32,36.50)};
\addplot[wapo gspo plot] coordinates {(1,27.57) (2,35.98) (4,43.60) (8,50.09) (16,55.82) (32,61.10)};
\addplot[wapo method plot] coordinates {(1,37.44) (2,45.42) (4,52.19) (8,57.93) (16,62.91) (32,67.5)};
\end{axis}
\end{tikzpicture} &
\begin{tikzpicture}
\begin{axis}[wapo eval axis, xmin=0, xmax=32, ymin=5, ymax=70, ylabel={}, xtick={1,2,4,8,16,32}]
\addplot[wapo base model plot] coordinates {(1,9.11) (2,14.41) (4,20.87) (8,28.10) (16,36.02) (32,44.30)};
\addplot[wapo drdapo plot] coordinates {(1,32.63) (2,41.43) (4,49.29) (8,56.03) (16,61.68) (32,66.40)};
\addplot[wapo grpo plot] coordinates {(1,31.27) (2,39.78) (4,47.56) (8,54.31) (16,60.07) (32,65.20)};
\addplot[wapo gspo plot] coordinates {(1,11.76) (2,16.26) (4,21.29) (8,26.59) (16,31.81) (32,36.70)};
\addplot[wapo method plot] coordinates {(1,29.59) (2,38.71) (4,46.87) (8,53.96) (16,60.29) (32,65.60)};
\end{axis}
\end{tikzpicture} &
\begin{tikzpicture}
\begin{axis}[wapo eval axis, xmin=0, xmax=32, ymin=0, ymax=46, ylabel={}, xtick={1,2,4,8,16,32}]
\addplot[wapo base model plot] coordinates {(1,1.46) (2,2.42) (4,3.76) (8,5.60) (16,8.19) (32,11.80)};
\addplot[wapo drdapo plot] coordinates {(1,4.67) (2,6.64) (4,9.13) (8,12.25) (16,15.96) (32,19.70)};
\addplot[wapo grpo plot] coordinates {(1,6.62) (2,10.08) (4,14.33) (8,19.10) (16,23.85) (32,28.10)};
\addplot[wapo gspo plot] coordinates {(1,10.50) (2,12.00) (4,13.47) (8,14.85) (16,16.13) (32,17.30)};
\addplot[wapo method plot] coordinates {(1,13.67) (2,19.40) (4,25.59) (8,31.77) (16,37.55) (32,42.60)};
\end{axis}
\end{tikzpicture} \\
\end{tabular}
\vspace{-6pt}
\caption{\textbf{HotpotQA and OTT-QA pass@$k$ curves.} Each sub-figure plots pass@$k$ EM score as a function of $k$ for one model family on the best evaluation checkpoint. \method\ not only outperforms baselines in pass@1, but also shows good pass@k results. \numina{} pass@$k$ results are shown in appendix.}
\label{fig:hotpot_ott_passk_curves}
\vspace{-4pt}
\end{figure*}

\label{sec:experiments}
We conduct extensive experiments across mathematical reasoning and multi-hop QA,
providing a well-rounded evaluation across different task formats and model
families. Specifically, we evaluate on \datamath{} and \numina{} for mathematical reasoning, and \hotpot{} and \ott{} for multi-hop QA. In the
single-turn math setting, the model directly produces a final answer, while in
the multi-turn QA setting, the model interacts with a search environment before
answering. We test three comparable-scale model families, \qwen{}, \smol{},
and \gemma{}, and compare \method{} against popular RLVR baselines,
including \GRPO{}, \DAPO{}, and \GSPO{} where each baseline is tuned to a strong-performing configuration under our experimental setup. For \GRPO{} we find $\epsilon=9.0$ works best, while for \DAPO{} we set $\epsilon_{\text{low}}=0.9$, $\epsilon_{\text{high}}=9.0$, and drop advantage normalization similar to \citet{liuunderstanding}. In the case of \GSPO{}, since it is sequence normalized it is more robust to outliers and we see the best results with $\epsilon = 0.2$. We evaluate performance using exact-match (EM)
accuracy, training stability, pass@k for measuring preservation of
exploratory behavior, and out-of-domain generalization. Dataset preparation and
training details are provided in Appendices~\ref{app:math-env},
\ref{app:qa-env}, and~\ref{app:hyperparams}.

\paragraph{\method{} improves stability across tasks and model families, with large gains on multi-hop QA and competitive math performance.}
Figure~\ref{fig:eval_curves_em} reports checkpoint performance across four datasets and three model families. \method{} remains stable across all evaluated settings, whereas \DAPO{} exhibits collapse in several cases, and GRPO collapses with \gemma{} on the \numina{} dataset. \GRPO{} and \GSPO{} are generally more stable than \DAPO{}, but often saturate early, especially on the multi-hop QA tasks. In contrast, \method{} remains stable and continues to improve through training, suggesting that removing non-positive advantage updates reduces a major source of RLVR instability without hurting learning capability.

\begin{table}[t]
\centering
\footnotesize
\setlength{\tabcolsep}{2.2pt}
\renewcommand{\arraystretch}{0.88}
\resizebox{\columnwidth}{!}{%
\begin{tabular}{@{}c@{\hspace{0.6em}}c@{}}
\begin{tabular}[t]{@{}lccc@{}}
\multicolumn{4}{@{}c@{}}{\textbf{(a) \hotpot $\rightarrow$ \wikidata}} \\
\toprule
Method & Qwen & Smol & Gemma \\
\midrule
\DAPO    & 13.3 & 30.3 & 22.0 \\
\GRPO    & 26.6 & 28.7 & 23.3 \\
\GSPO    & 29.7 & 29.8 & 24.4 \\
\method & \textbf{34.2} & \textbf{31.4} & \textbf{25.6} \\
\bottomrule
\end{tabular}
&
\begin{tabular}[t]{@{}lccc@{}}
\multicolumn{4}{@{}c@{}}{\textbf{(b) Numina $\rightarrow$ \aime}} \\
\toprule
Method & Qwen & Smol & Gemma \\
\midrule
\DAPO    & 39.06 & \textbf{27.18} & 13.75 \\
\GRPO    & \textbf{45.10} & 25.93 & 13.54 \\
\GSPO    & 42.08 & 26.04 & \textbf{14.27} \\
\method & 38.43 & 26.77 & 12.71 \\
\bottomrule
\end{tabular}
\end{tabular}
}
\caption{
\textbf{Out-of-domain evaluation.}
Best-checkpoint transfer from source to target datasets. We report exact match for \wikidata{} and average@32 for \aime{}.
}
\label{tab:ood_results}
\vspace{-10pt}
\end{table}

On \ott{} and \hotpot{}, \method{} consistently achieves strong performance across model families: on the \ott{} dataset, \method{} outperforms the next best stable baseline by 9.9\% for \qwen{} and 3.2\% for \gemma{}; similarly on the \hotpot{} dataset, the margins are 4.5\% with \qwen{} and 10.6\% with \gemma{} respectively. On \datamath{} and \numina{}, \method{} nearly matches other baselines, while slightly lagging on early checkpoints. One possible explanation is the adaptive factor $1-q_x$ used by \method{}, which emphasizes lower-success prompts and can yield more conservative updates early in training. This pattern is visible on \numina{} in Figure~\ref{fig:eval_curves_em}: both \smol{} and \gemma{} initially trail the baselines, but close the gap at later training steps. Overall, \method{} appears to be most beneficial on more difficult datasets, as indicated by lower EM score ranges across methods. For example \method{} performs best on \ott{} (hardest) and remains largely comparable to the strongest baselines on \datamath{} (easiest). Table \ref{tab:main_results} in Appendix \ref{app:supp_exp} includes exact numbers for the plots in \Cref{fig:eval_curves_em}.

\paragraph{\method{} preserves exploration.}
A potential concern is that removing negative-advantage updates may over-exploit
already discovered trajectories and reduce sampling diversity.  We report \hotpot{} and \ott{} pass@\(k\) curves at the best evaluated
checkpoint in Figure~\ref{fig:hotpot_ott_passk_curves}. \method{} achieves the best pass@1 and also maintains the strongest
pass@\(k\) performance across all three model families. We also observe competitive performance at higher pass@$k$ values on both
\datamath{} and \numina{}. For example, on \numina{} with \smol{},
\method{} achieves $81.02\%$ pass@$16$, matching \GSPO{}'s $81.0\%$. (see Appendix \ref{app:supp_exp}, \Cref{fig:math_numina_passk_curves_no_base}).

\begin{table}[t]
\vspace{-3pt}
\centering
\scriptsize
\setlength{\tabcolsep}{2.8pt}
\renewcommand{\arraystretch}{0.92}
\resizebox{\columnwidth}{!}{%
\begin{tabular}{llrrrrrr}
\toprule
Model & Method & p@1 & p@2 & p@4 & p@8 & p@16 & p@32 \\
\midrule
\qwen  & \GRPO     & \textbf{45.1} & \textbf{54.9} & \textbf{61.8} & 65.9 & 68.4 & 70.0 \\
\qwen  & \DAPO     & 39.1 & 47.5 & 55.2 & 62.0 & 69.7 & \textbf{80.0} \\
\qwen  & \GSPO     & 42.1 & 51.9 & 60.6 & \textbf{66.7} & 71.3 & 76.7 \\
\qwen  & \method{} & 36.4 & 45.1 & 54.3 & 63.8 & \textbf{71.5} & 76.7 \\
\midrule
\smol  & \GRPO     & 25.9 & 31.5 & 38.2 & 45.8 & 53.0 & 60.0 \\
\smol  & \DAPO     & \textbf{27.2} & 33.4 & 40.5 & 47.9 & 53.9 & 56.7 \\
\smol  & \GSPO     & 26.0 & 32.6 & 40.9 & 50.2 & 59.1 & \textbf{66.7} \\
\smol  & \method{} & 25.9 & \textbf{33.5} & \textbf{42.9} & \textbf{53.5} & \textbf{61.7} & \textbf{66.7} \\
\midrule
\gemma & \GRPO     & 13.5 & 16.6 & 20.3 & 24.4 & \textbf{28.9} & \textbf{33.3} \\
\gemma & \DAPO     & 13.8 & 17.1 & 20.5 & 23.5 & 26.3 & 30.0 \\
\gemma & \GSPO     & \textbf{14.3} & \textbf{17.6} & \textbf{21.0} & \textbf{24.3} & 27.4 & 30.0 \\
\gemma & \method{} & 12.4 & 15.9 & 19.6 & 23.4 & 25.5 & 30.0 \\
\bottomrule
\end{tabular}%
}
\vspace{-4pt}
\caption{Out-of-domain transfer from NuminaMath-LEAN to \aime{} pass@\(k\) EM-results.}
\label{tab:aime-passk}
\vspace{-8pt}
\end{table}

\paragraph{\method{} shows robust out-of-domain transfer.}
Table~\ref{tab:ood_results} evaluates out-of-domain transfer from the training
distribution to a held-out target dataset. On \hotpot{}$\rightarrow$\wikidata{},
\method{} achieves the best performance across all three model families,
indicating that its gains on multi-hop QA are not specific to the training
benchmark. On NuminaMath-LEAN$\rightarrow$\aime{}, \method{} remains competitive
with RLVR baselines. The pass@$k$ results in Table~\ref{tab:aime-passk} show a
similar trend to the in-domain setting: although \method{} is not always strongest
at low $k$, it becomes increasingly competitive as $k$ grows, especially for
\qwen{} and \smol{}.

\section{Conclusion}
We presented a token-level gradient analysis of RLVR instability and introduced a peak--valley taxonomy for understanding how advantage-weighted updates shape
stability and exploration. Building on this analysis, we proposed \method{}, a
minimal \GRPO{}-style objective that removes non-positive advantage updates while
retaining online rollouts, group-normalized advantages, importance ratios, and
clipping. Across mathematical reasoning and multi-hop QA benchmarks, \method{}
improves training stability, preserves pass@\(k\), and outperforms strong RLVR
baselines in regimes where existing methods collapse or saturate early, while
remaining competitive in settings where those baselines are stable. Future work includes extending this study to additional verifiable domains, including programming tasks such as text-to-SQL~\citep{pourreza2024dts, gorti2025msc} and code generation~\citep{le2022coderl, chencodet}, as well as to larger model scales and MoE architectures.

\bibliography{custom}

\appendix

\clearpage

\section{Additional Collapse Examples}\label{app:additional_clip_eps_exps}

\pgfplotstableread[col sep=comma]{plots/data/ott_qwen_grpo_eps_9_0.csv}{\ottQwenGrpoNine}
\pgfplotstableread[col sep=comma]{plots/data/ott_qwen_grpo_eps_0_2.csv}{\ottQwenGrpoPointTwo}
\pgfplotstableread[col sep=comma]{plots/data/ott_gemma_drdapo_10_0_8.csv}{\ottGemmaDrDapoPointEight}
\pgfplotstableread[col sep=comma]{plots/data/ott_gemma_drdapo_10_0_1.csv}{\ottGemmaDrDapoPointOne}

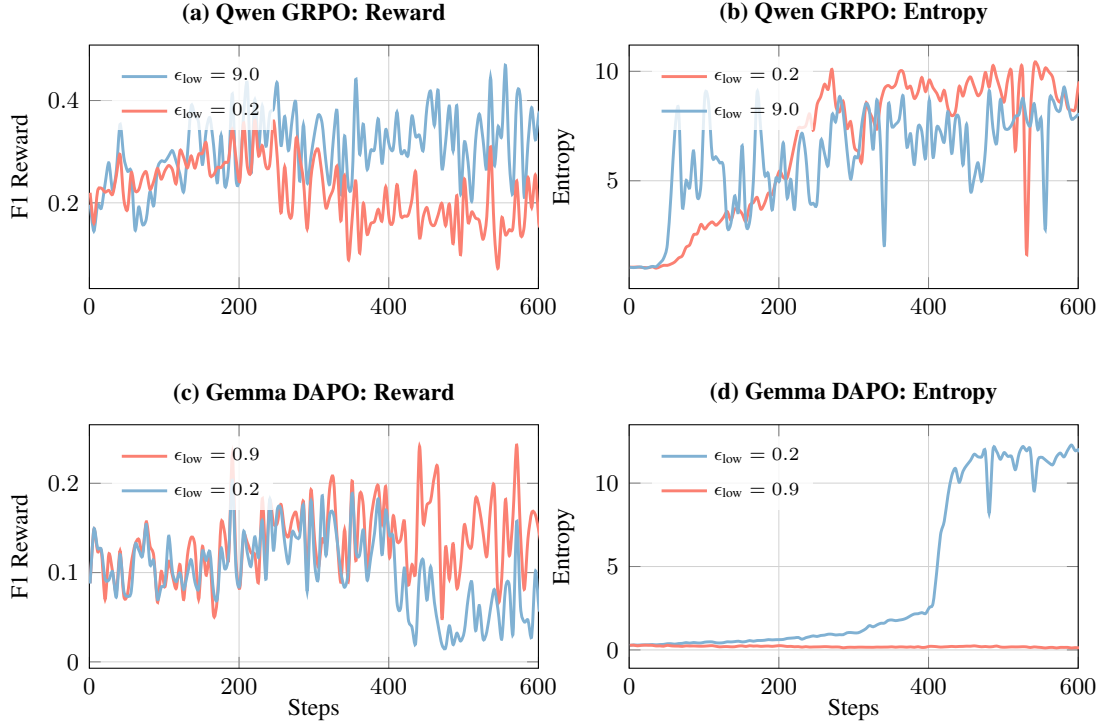
\begin{figure*}[t!]
\centering
\begin{tikzpicture}
    \definecolor{cottqwengrpoepsnine}{RGB}{128,177,211}
    \definecolor{cottqwengrpoepspointtwo}{RGB}{251,128,114}
    \definecolor{cottgemmadrdapopointeight}{RGB}{128,177,211}
    \definecolor{cottgemmadrdapopointone}{RGB}{251,128,114}

    \begin{groupplot}[
        group style={
            group size=2 by 2,      
            horizontal sep=1.2cm,
            vertical sep=1.8cm,     
        },
        width=0.47\linewidth,
        height=0.30\linewidth,
        xmin=0,
        xmax=600,
        grid=both,
        grid style={line width=.1pt, draw=gray!25},
        major grid style={line width=.2pt, draw=gray!35},
        label style={font=\small},
        tick label style={font=\footnotesize},
        title style={font=\small\bfseries, yshift=-2pt},
        filter discard warning=false,
    ]

        \nextgroupplot[
            ylabel={F1 Reward},
            title={(a) Qwen GRPO: Reward},
            legend style={
                font=\scriptsize, draw=none, fill=white, fill opacity=0.8,
                text opacity=1, at={(0.05, 0.95)}, anchor=north west,
                legend columns=1, row sep=0.1cm
            }
        ]
        \addplot[color=cottqwengrpoepsnine, line width=1.1pt, no marks, smooth, each nth point=5]
            table [x=step, y=f1_reward] {\ottQwenGrpoNine};
        \addlegendentry{$\epsilon_{\text{low}}=9.0$}
        
        \addplot[color=cottqwengrpoepspointtwo, line width=1.1pt, no marks, smooth, each nth point=5]
            table [x=step, y=f1_reward] {\ottQwenGrpoPointTwo};
        \addlegendentry{$\epsilon_{\text{low}}=0.2$}

        \nextgroupplot[
            ylabel={Entropy},
            title={(b) Qwen GRPO: Entropy},
            legend style={
                font=\scriptsize, draw=none, fill=white, fill opacity=0.8,
                text opacity=1, at={(0.05, 0.95)}, anchor=north west,
                legend columns=1, row sep=0.1cm
            }
        ]
        \addplot[color=cottqwengrpoepspointtwo, line width=1.1pt, no marks, smooth, each nth point=5]
            table [x=step, y=entropy] {\ottQwenGrpoPointTwo};
        \addlegendentry{$\epsilon_{\text{low}}=0.2$}
        
        \addplot[color=cottqwengrpoepsnine, line width=1.1pt, no marks, smooth, each nth point=5]
            table [x=step, y=entropy] {\ottQwenGrpoNine};
        \addlegendentry{$\epsilon_{\text{low}}=9.0$}

        \nextgroupplot[
            ylabel={F1 Reward},
            xlabel={Steps},
            x label style={yshift=0.20cm},
            title={(c) Gemma DAPO: Reward},
            legend style={
                font=\scriptsize, draw=none, fill=white, fill opacity=0.8,
                text opacity=1, at={(0.05, 0.95)}, anchor=north west,
                legend columns=1, row sep=0.1cm
            }
        ]
        \addplot[color=cottgemmadrdapopointone, line width=1.1pt, no marks, smooth, each nth point=5]
            table [x=step, y=f1_reward] {\ottGemmaDrDapoPointOne};
        \addlegendentry{$\epsilon_{\text{low}}=0.9$}
        
        \addplot[color=cottgemmadrdapopointeight, line width=1.1pt, no marks, smooth, each nth point=5]
            table [x=step, y=f1_reward] {\ottGemmaDrDapoPointEight};
        \addlegendentry{$\epsilon_{\text{low}}=0.2$}

        \nextgroupplot[
            ylabel={Entropy},
            xlabel={Steps},
            x label style={yshift=0.20cm},
            title={(d) Gemma DAPO: Entropy},
            legend style={
                font=\scriptsize, draw=none, fill=white, fill opacity=0.8,
                text opacity=1, at={(0.05, 0.95)}, anchor=north west,
                legend columns=1, row sep=0.1cm
            }
        ]
        \addplot[color=cottgemmadrdapopointeight, line width=1.1pt, no marks, smooth, each nth point=5]
            table [x=step, y=entropy] {\ottGemmaDrDapoPointEight};
        \addlegendentry{$\epsilon_{\text{low}}=0.2$}
        
        \addplot[color=cottgemmadrdapopointone, line width=1.1pt, no marks, smooth, each nth point=5]
            table [x=step, y=entropy] {\ottGemmaDrDapoPointOne};
        \addlegendentry{$\epsilon_{\text{low}}=0.9$}

    \end{groupplot}
\end{tikzpicture}

\caption{F1-reward and entropy trajectories on OTT-QA under varying negative clipping thresholds. \textbf{Top (\subref{fig:qwen_reward}, \subref{fig:qwen_entropy}):} \qwen{} trained with GRPO. \textbf{Bottom (\subref{fig:gemma_reward}, \subref{fig:gemma_entropy}):} \gemma{} trained with DAPO. Larger $\epsilon_{\text{low}}$ indicates more aggressive clipping (red), which surpresses divergent negative tokens, but does not always improve training stability while reducing entropy.}\label{fig:qwen_gemma_clip_exps}

\phantomsubcaption\label{fig:qwen_reward}
\phantomsubcaption\label{fig:qwen_entropy}
\phantomsubcaption\label{fig:gemma_reward}
\phantomsubcaption\label{fig:gemma_entropy}

\end{figure*}
We show additional examples of aggressive clipping causing collapse in \Cref{fig:qwen_gemma_clip_exps}. As compared to \Cref{fig:dapo_smol_analysis}, here we use \qwen{}~\cite{yang2025qwen3} trained with GRPO~\cite{deng2025grpo}, and \gemma{}~\cite{gemmateam2025gemma3technicalreport} trained with DAPO~\cite{yu2025dapo}. In both settings, more aggressive clipping suppresses a larger fraction of divergent tokens but does not improve stability. This highlights that RLVR instability is not explained solely by the trust region.

\section{Proofs for the Gradient-Motivated Taxonomy}
\label{app:gradient-token-taxonomy}

\subsection{First-Order Gradient Prediction}

Here we derive the result in \Cref{eq:delta_prob}.

\textbf{Fact:} Denote logits as $z$, the softmax probability of the sampled token as $p_s$, and similarly $p_i$ for a non-sampled token. A small gradient-descent step with step size $\eta>0$ on the advantage-weighted negative log-likelihood $\ell_s(z)=-A\log p_s$ changes the probability of $p_i$ as
\begin{equation}
    \Delta p_i = \eta A p_i\left(C(p)-p_s-p_i\right) +O(\eta^2),
\end{equation}
where $C(p)=\sum_{j=1}^{V}p_j^2$.

\textit{Proof.} Let $z$ denote the initial logits and $z'$ denote the updated logits after a single gradient descent step. The gradient of $\ell_s(z)$ with respect to an arbitrary logit $z_k$ is computed as:
\begin{equation}
    \frac{\partial \ell_s}{\partial z_k} = -A \frac{1}{p_s} \frac{\partial p_s}{\partial z_k} = A (p_k - \delta_{sk}),
\end{equation}
where $\delta_{sk}$ is the Kronecker delta ($\delta_{sk}=1$ if $s=k$, and $0$ otherwise). Applying the gradient descent update with step size $\eta$, the change in the logits is:
\begin{equation}
    \Delta z_k = z'_k - z_k = -\eta A (p_k - \delta_{sk}).
\end{equation}
To find the resulting change in the probability distribution $\Delta p_i$ for a non-sampled token $i \neq s$, we perform a first-order Taylor expansion of $p_i(z)$ around $z$:
\begin{equation}
    \Delta p_i = \sum_{k \in [V]} \frac{\partial p_i}{\partial z_k} \Delta z_k + O(\eta^2).
\end{equation}
Recall that the Jacobian of the softmax function is given by $\frac{\partial p_i}{\partial z_k} = p_i(\delta_{ik} - p_k)$. Substituting this and our logit update $\Delta z_k$ into the expansion yields:
\begin{equation}
\begin{split}
    \Delta p_i &= \sum_{k \in [V]} p_i(\delta_{ik} - p_k) \big(-\eta A (p_k - \delta_{sk})\big) \\
    &\quad + O(\eta^2) \\
    &= -\eta A p_i \sum_{k \in [V]} (\delta_{ik} - p_k)(p_k - \delta_{sk}) \\
    &\quad + O(\eta^2).
\end{split}
\end{equation}
We can distribute the terms inside the summation:
\begin{equation}
\begin{split}
    \sum_{k \in [V]} &(\delta_{ik} p_k - \delta_{ik}\delta_{sk} - p_k^2 + p_k\delta_{sk}) \\
    &= \sum_{k} \delta_{ik} p_k - \sum_{k} \delta_{ik}\delta_{sk} \\
    &\quad - \sum_{k} p_k^2 + \sum_{k} p_k\delta_{sk}.
\end{split}
\end{equation}
Since we are evaluating this for a non-sampled token ($i \neq s$), the term $\delta_{ik}\delta_{sk}$ is uniformly zero for all $k$. We define the sum of squared probabilities as $C(p) = \sum_{k \in [V]} p_k^2$. Evaluating the remaining sums, we obtain:
\begin{equation}
    p_i - 0 - C(p) + p_s.
\end{equation}
Substituting this simplified summation back into our equation for $\Delta p_i$ provides the final result:
\begin{equation}
\begin{split}
    \Delta p_i &= -\eta A p_i \big(p_i - C(p) + p_s\big) + O(\eta^2) \\
    &= \eta A p_i \big(C(p) - p_s - p_i\big) + O(\eta^2),
\end{split}
\end{equation}
which completes the proof. \hfill \qedsymbol

\subsection{Extremal Tokens and the Reference Level}

Next we justify the order properties described in \Cref{sec:gradient-token-taxonomy}.

\textbf{Fact:} The reference level $C(p) = \sum_{k \in [V]} p_k^2$ satisfies the following order properties: (1) $p_{\max} \ge C(p)$ 
and (2) $p_{\min} \le C(p)$ with equality only for a uniform distribution.

\textit{Proof.} Because $p$ is a valid probability distribution, we have $\sum_{k \in [V]} p_k = 1$ and $p_k \ge 0$ for all $k$. We prove each of the three properties in turn.

\textbf{Property 1: $p_{\max} \ge C(p)$.} \\
By definition, $p_k \le p_{\max}$ for all $k \in [V]$. Multiplying both sides by the non-negative probability $p_k$ yields $p_k^2 \le p_{\max} p_k$. Summing this inequality over the entire vocabulary gives:
\begin{equation}
\begin{split}\label{eq:prop_1}
    C(p) &= \sum_{k \in [V]} p_k^2 \le \sum_{k \in [V]} p_{\max} p_k \\
    &= p_{\max} \sum_{k \in [V]} p_k = p_{\max}.
\end{split}
\end{equation}
Thus, $p_{\max} \ge C(p)$. It follows immediately that any token attaining the maximum probability satisfies this condition, making it a $\peak$.


\textbf{Property 2: $p_{\min} \le C(p)$.} \\
Applying the same logic as the Property 1, we note that $p_k \ge p_{\min}$ for all $k$ in the support of $p$. Multiplying by $p_k$ and summing over $k$ yields:
\begin{equation}
\begin{split}
    C(p) &= \sum_{k \in [V]} p_k^2 \ge \sum_{k \in [V]} p_{\min} p_k \\
    &= p_{\min} \sum_{k \in [V]} p_k = p_{\min}.
\end{split}
\end{equation}
This establishes that $p_{\min} \le C(p)$. For equality to hold, $C(p) - p_{\min} = 0$, which implies:
\begin{equation}
    \sum_{k \in [V]} p_k(p_k - p_{\min}) = 0.
\end{equation}
Since $p_k \ge 0$ and $p_k - p_{\min} \ge 0$ for all $k$, every term in the summation must independently be zero. This requires $p_k = p_{\min}$ for every token in the support, which only occurs when the distribution $p$ is strictly uniform. Therefore, a minimum-probability token is always a $\valley$, except when the distribution is perfectly uniform. \hfill \qedsymbol

\subsection{First-Order Entropy Direction}
\label{app:entropy-direction}

Below we derive the change in entropy from \Cref{eq:delta_entropy}.

\textit{Proof of $\Delta H$ Update.} Let $H(p) = -\sum_{k \in [V]} p_k \log p_k$ be the entropy of the next-token distribution. The first-order Taylor expansion of the entropy change is:
\begin{equation}
    \Delta H = \sum_{k \in [V]} \frac{\partial H}{\partial p_k} \Delta p_k + O(\eta^2).
\end{equation}
Taking the derivative of $H(p)$, we have $\frac{\partial H}{\partial p_k} = -1 - \log p_k$. Because probability mass must be conserved ($\sum_k p_k = 1$), the sum of the changes is zero: $\sum_{k \in [V]} \Delta p_k = 0$. Thus, the entropy update simplifies to:
\begin{equation}
    \Delta H = -\sum_{k \in [V]} \Delta p_k \log p_k + O(\eta^2).
\end{equation}

From our earlier derivation, the general update for any token $k$ (including the sampled token $s$) is $\Delta p_k = \eta A p_k (C(p) - p_s - p_k + \delta_{sk})$. Substituting this into our simplified expression for $\Delta H$ yields:
\begin{equation}
\begin{split}
    \Delta H &= -\eta A \sum_{k \in [V]} p_k \log p_k \big(C(p) - p_s - p_k + \delta_{sk}\big) \\
    &\quad + O(\eta^2).
\end{split}
\end{equation}
We distribute the summation over the terms inside the parentheses to get:
\begin{equation}
\begin{split}
    \Delta H &= -\eta A \Bigg[ (C(p) - p_s) \sum_{k \in [V]} p_k \log p_k \\
    &\quad - \sum_{k \in [V]} p_k^2 \log p_k \\
    &\quad + \sum_{k \in [V]} p_k \log p_k \delta_{sk} \Bigg] + O(\eta^2).
\end{split}
\end{equation}
Recognizing that $\sum_{k} p_k \log p_k = -H(p)$ and applying the Kronecker delta to the final term, we arrive at the final expression after rearrangement:
\begin{equation}\label{eq:delta_entropy_proof}
\begin{split}
    \Delta H &= -\eta A \Bigg[ p_s \log p_s - \sum_{k \in [V]} p_k^2 \log p_k \\
    &\quad + H(p)\big(p_s - C(p)\big) \Bigg] + O(\eta^2).
\end{split}
\end{equation}
\hfill \qedsymbol

\vspace{1em}

\textit{Proof of Entropy Direction.} Let $B(p_s)$ denote the bracketed term in Equation \ref{eq:delta_entropy_proof}. In \Cref{sec:gradient-token-taxonomy} we described the behaviour of $B(p_s)$ for the token types defined in \Cref{eq:peak_valley_def}. Here we provide the complete reasoning by computing its sign for $\valley$ and $\peak$ tokens. For convenience, we define $M(p) = -\sum_{k \in [V]} p_k^2 \log p_k$ and the continuous function $h(x) = -x \log x$.

\textbf{Direction for $\valley$ tokens ($p_s < C(p)$).} \\
We extend the bracketed term to a function over $x \in [0, 1]$:
\begin{equation}
    B(x) = M(p) - h(x) + H(p)\big(x - C(p)\big).
\end{equation}
Note that we consider this as a function of variable $x$ with $p$ representing a fixed probability distribution, such that $M(p)$, $H(p)$, and $C(p)$ are independent of $x$. Then the second derivative is $B''(x) = 1/x > 0$, implying $B(x)$ is strictly convex on $(0, 1]$. By convexity, for any $x \in [0, C(p)]$ (noting that $C(p)\leq 1$), $B(x)$ is bounded above by the convex combination of its endpoints. Letting $\lambda = x / C(p)$, we have:
\begin{equation}\label{eq:convexity}
    B(x) \le (1 - \lambda) B(0) + \lambda B(C(p)).
\end{equation}
We evaluate both endpoints to show they are non-positive. At $x=0$, we observe:
\begin{equation}
\begin{split}
    B(0) &= M(p) - H(p)C(p) \\
    &= \mathbb{E}_{k \sim p}[-p_k \log p_k] \\
    &\quad - \mathbb{E}_{k \sim p}[-\log p_k]\mathbb{E}_{k \sim p}[p_k] \\
    &= \text{Cov}_{k \sim p}(p_k, -\log p_k) \le 0,
\end{split}
\end{equation}
which holds because $-\log$ is a monotonically decreasing function. 

At $x=C(p)$, we define a valid probability distribution $q$ where $q_k = p_k^2 / C(p)$. Then:
\begin{equation}
\begin{split}
    B(C(p)) &= M(p) - h(C(p)) \\
    &= -\sum_{k \in [V]} p_k^2 \log \frac{p_k}{C(p)} \\
    &= -C(p) \sum_{k \in [V]} q_k \log \frac{q_k}{p_k} \\
    &= -C(p) D_{\mathrm{KL}}(q \| p) \le 0,
\end{split}
\end{equation}
Since the KL divergence is non-negative.

Having established that both endpoints of $B(x)$ are non-positive, by \Cref{eq:convexity} we find that $B(x) \le 0$ for all $x \in [0, C(p)]$. Since a $\valley$ token satisfies $p_s < C(p)$, it immediately follows that $B(p_s) \le 0$.

\textbf{Direction for Maximum-Probability $\peak$ tokens.} \\
Consider a token $s$ where $p_s = \max_k p_k$. From Property 1 in \Cref{eq:prop_1}, we know $C(p) \le p_s$. We rewrite the bracket $B(p_s)$ by extracting terms:
\begin{equation}
\begin{split}
    B(p_s) &= \sum_{k \in [V]} p_k \big(p_k + p_s - C(p)\big)(-\log p_k) \\
    &\quad - p_s(-\log p_s).
\end{split}
\end{equation}
Notice that the weights $w_k = p_k(p_k + p_s - C(p))$ are strictly non-negative and sum exactly to $p_s$:
\begin{equation}
\begin{split}
    \sum_{k \in [V]} w_k &= \sum_{k} p_k^2 + p_s \sum_{k} p_k - C(p) \sum_{k} p_k \\
    &= C(p) + p_s - C(p) = p_s.
\end{split}
\end{equation}
Furthermore, because $p_s$ is the maximum probability, $-\log p_k \ge -\log p_s$ for all $k$. Applying this inequality directly to our expanded bracket yields:
\begin{equation}
\begin{split}
    B(p_s) &= \sum_{k \in [V]} w_k (-\log p_k) - p_s(-\log p_s) \\
    &\ge \sum_{k \in [V]} w_k (-\log p_s) - p_s(-\log p_s) \\
    &= p_s(-\log p_s) - p_s(-\log p_s) = 0.
\end{split}
\end{equation}
Thus, $B(p_s) \ge 0$ for maximum-probability tokens. 

\textit{Remark:} A $\peak$ token that is not the global maximum may still yield a negative $B(p_s)$. For example, if $p = (0.14, 0.46, 0.40)$ and the sampled token is $p_s = 0.40$, then $C(p) \approx 0.391$. While $p_s > C(p)$ (making it a $\peak$), computing the bracket gives $B(p_s) \approx -0.0083 < 0$. However, these corner cases are rare in practice and our experiments follow the expected trend that $\Pos$-$\mathrm{peak}$ tokens decrease entropy while $\Neg$-$\mathrm{peak}$ tokens increase entropy. \hfill \qedsymbol
\section{Additional Token Dynamics Plots}
\label{app:additional-token-dynamics}
\pgfplotsset{compat=1.18}

\definecolor{qwenNegValley}{RGB}{255, 150, 150}
\definecolor{qwenNegPeak}{RGB}{255, 100, 100}
\definecolor{qwenPosValley}{RGB}{140, 220, 140}
\definecolor{qwenPosPeak}{RGB}{70, 180, 70}

\definecolor{qwenRunValley}{RGB}{0, 168, 181}
\definecolor{qwenRunPeak}{RGB}{100, 110, 120}
\definecolor{qwenRunPos}{RGB}{44, 160, 44}
\definecolor{qwenRunNeg}{RGB}{210, 115, 220}
\definecolor{qwenRunDrdapo}{RGB}{219, 0, 117}
\pgfplotstableread[col sep=comma]{plots/data/numina_qwen_neg_challenger.csv}\qwenDataNegValley
\pgfplotstableread[col sep=comma]{plots/data/numina_qwen_neg_favorite.csv}\qwenDataNegPeak
\pgfplotstableread[col sep=comma]{plots/data/numina_qwen_pos_challenger.csv}\qwenDataPosValley
\pgfplotstableread[col sep=comma]{plots/data/numina_qwen_pos_favorite.csv}\qwenDataPosPeak

\pgfplotstableread[col sep=comma]{plots/data/numina_qwen_valley.csv}\qwenRunValleyData
\pgfplotstableread[col sep=comma]{plots/data/numina_qwen_peak.csv}\qwenRunPeakData
\pgfplotstableread[col sep=comma]{plots/data/numina_qwen_pos.csv}\qwenRunPosData
\pgfplotstableread[col sep=comma]{plots/data/numina_qwen_neg.csv}\qwenRunNegData
\pgfplotstableread[col sep=comma]{plots/data/numina_qwen_dr_dapo-10-0.1.csv}\qwenRunDrDapo

\begin{figure*}[htbp]
\centering
\begin{tikzpicture}

    \begin{groupplot}[
        group style={
            group size=2 by 2,
            horizontal sep=1.8cm,
            vertical sep=1.7cm,
        },
        width=0.46\linewidth,
        height=0.22\linewidth,
        grid=both,
        grid style={line width=.1pt, draw=gray!25},
        major grid style={line width=.2pt, draw=gray!35},
        label style={font=\small},
        tick label style={font=\footnotesize},
        title style={font=\small\bfseries, yshift=-2pt},
        filter discard warning=false
    ]

        \nextgroupplot[
            restrict x to domain=0:60,
            xmin=0, xmax=60,
            ylabel={Reward},
            title={(a) Reward Optimization via Peak-Valley Dynamics},
            legend style={
                font=\scriptsize, draw=none, fill=none,
                at={(1.12, -0.25)},
                anchor=north,
                legend columns=4,
                /tikz/every even column/.append style={column sep=0.4cm}
            }
        ]
        \addplot[color=qwenNegValley, line width=1.2pt, no marks, dashed, dash pattern=on 3pt off 2pt] table [x index=0, y index=1] {\qwenDataNegValley};
        \addlegendentry{Neg-Valley}

        \addplot[color=qwenNegPeak, line width=1.2pt, no marks, solid] table [x index=0, y index=1] {\qwenDataNegPeak};
        \addlegendentry{Neg-Peak}

        \addplot[color=qwenPosValley, line width=1.2pt, no marks, dashed, dash pattern=on 3pt off 2pt] table [x index=0, y index=1] {\qwenDataPosValley};
        \addlegendentry{Pos-Valley}

        \addplot[color=qwenPosPeak, line width=1.2pt, no marks, solid] table [x index=0, y index=1] {\qwenDataPosPeak};
        \addlegendentry{Pos-Peak}

        \nextgroupplot[
            restrict x to domain=0:60,
            xmin=0, xmax=60,
            ylabel={Entropy},
            title={(b) Exploration Decay Profiles}
        ]
        \addplot[color=qwenPosValley, line width=1.2pt, no marks, dashed, dash pattern=on 3pt off 2pt] table [x index=0, y index=2] {\qwenDataPosValley};
        \addplot[color=qwenPosPeak, line width=1.2pt, no marks, solid] table [x index=0, y index=2] {\qwenDataPosPeak};
        \addplot[color=qwenNegValley, line width=1.2pt, no marks, dashed, dash pattern=on 3pt off 2pt] table [x index=0, y index=2] {\qwenDataNegValley};
        \addplot[color=qwenNegPeak, line width=1.2pt, no marks, solid] table [x index=0, y index=2] {\qwenDataNegPeak};

        \nextgroupplot[hide axis]

        \nextgroupplot[
            restrict x to domain=0:300,
            xmin=0, xmax=300,
            width=0.75\linewidth,
            height=0.27\linewidth,
            xshift=-0.24\linewidth,
            ylabel={Reward},
            title={(c) Integrated Trajectory Comparison},
            legend style={
                font=\scriptsize,
                draw=none,
                fill=none,
                at={(0.5, -0.28)},
                anchor=north,
                legend columns=5,
                /tikz/every even column/.append style={column sep=0.4cm}
            }
        ]

        \addplot[color=qwenRunPeak, line width=1.2pt, no marks, smooth, each nth point=3] table [x index=0, y index=1] {\qwenRunPeakData};
        \addlegendentry{Peak}

        \addplot[color=qwenRunPos, line width=1.2pt, no marks, smooth, each nth point=3] table [x index=0, y index=1] {\qwenRunPosData};
        \addlegendentry{Pos}

        \addplot[color=qwenRunNeg, line width=1.2pt, no marks, smooth, each nth point=3] table [x index=0, y index=1] {\qwenRunNegData};
        \addlegendentry{Neg}

        \addplot[color=qwenRunDrdapo, line width=1.4pt, no marks, dashed, dash pattern=on 4pt off 2pt, smooth, each nth point=3] table [x index=0, y index=1] {\qwenRunDrDapo};
        \addlegendentry{DAPO}

        \addplot[color=qwenRunValley, line width=1.2pt, no marks, smooth, each nth point=3] table [x index=0, y index=1] {\qwenRunValleyData};
        \addlegendentry{Valley}

    \end{groupplot}
\end{tikzpicture}

\caption{We train \qwen{} on \numina{} with one group of our taxonomy at a time. Panel (\subref{fig:qwen_ablation_entropy}) validates \Cref{eq:delta_entropy}: $\Pos$-\(\valley\) and $\Neg$-\(\peak\) tokens increase entropy, while $\Pos$-\(\peak\) and $\Neg$-\(\valley\) tokens decrease entropy. Entropy-increasing groups collapse immediately in (\subref{fig:qwen_ablation_reward}). $\Pos$-\(\peak\) is stable but does not match the best integrated trajectory in (\subref{fig:qwen_ablation_trajectory}), motivating objectives that balance exploration and exploitation. We observed that the entropy of \qwen{} is lower than \smol{}. This could explain why $\Neg$-\(\valley\) did not collapse immediately. In contrast, \(\valley\)-only performs on-par with the baseline DAPO in this setting, showing that it is a promising candidate for future study. $\Pos$-only remains competitive in all settings.}
\label{fig:qwen_ablation_matrix}

\phantomsubcaption\label{fig:qwen_ablation_reward}
\phantomsubcaption\label{fig:qwen_ablation_entropy}
\phantomsubcaption\label{fig:qwen_ablation_trajectory}

\end{figure*}
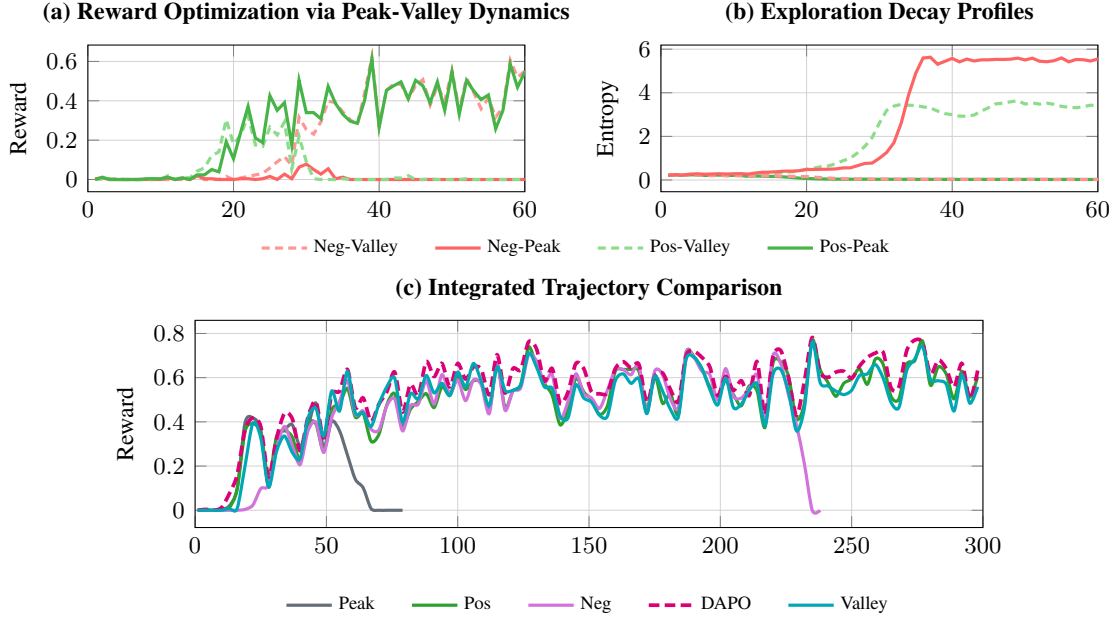

We include additional entropy and reward plots for \numina{} runs on \qwen{}. These plots confirm the entropy direction predictions from \Cref{sec:gradient-token-taxonomy} in \Cref{{fig:qwen_ablation_matrix}}: $\pos$-\(\valley\) and $\nneg$-\(\peak\) updates increase entropy, while $\pos$-\(\peak\) and $\nneg$-\(\valley\) updates decrease entropy.

\section{WAPO Gradient Analysis}\label{app:wapo_analysis}

We begin by deriving \Cref{eq:wapo_success_gradient} from \Cref{sec:wapo}.
\begin{proposition}
We aim to prove that for a policy model $\pi_\theta$ operating under a binary reward setting ($r_i \in \{0,1\}$), the gradient of the probability $q_x$ of generating a correct answer can be empirically estimated over $G$ rollouts as:
\begin{equation}\label{eq:policy_gradient_app}
    \nabla q_x \approx \frac{1}{G} \sum_{i=1}^{G} \sum_{j=1}^{T_i} r_i\nabla \log p_{ij}.
\end{equation}
\end{proposition}

\begin{proof}
Let $y_i = (y_{i1}, y_{i2}, \dots, y_{iT_i})$ denote a sequence of $T_i$ tokens generated by the policy model $\pi_\theta$ in response to a prompt $x$. The joint probability of generating the entire sequence $y_i$ is the product of the individual step-wise token probabilities:
\begin{equation}
    P(y_i) = \prod_{j=1}^{T_i} \pi_{\theta}(y_{ij} \mid y_{i,j-1}\dots y_{i1}, x) = \prod_{j=1}^{T_i} p_{ij}.
\end{equation}
Because the reward $r(x, y_i)$ (which we abbreviate as $r_i$) is binary, taking values in $\{0, 1\}$, the probability $q_x$ that the model correctly answers the prompt is exactly equivalent to the expected reward over all possible generated sequences:
\begin{equation}
    q_x = \mathbb{E}_{y_i \sim \pi_\theta} [r_i] = \sum_{y_i} r_i P(y_i).
\end{equation}

We compute the gradient of $q_x$ with respect to the model parameters $\theta$. Since the reward function evaluates the correctness of a sequence and does not intrinsically depend on $\theta$, the gradient operator only applies to the sequence probability:
\begin{equation}\label{eq:policy-gradient-seq}
    \nabla q_x = \sum_{y_i} r_i\nabla P(y_i).
\end{equation}
We apply the standard log-derivative trick, $\nabla P(y_i) = P(y_i) \nabla \log P(y_i)$, to rewrite this summation back into the form of an expectation:
\begin{equation}
\begin{aligned}
    \nabla q_x &= \sum_{y_i} r_i P(y_i) \nabla \log P(y_i)  \\
    &= \mathbb{E}_{y_i \sim \pi_\theta} \big[ r_i\nabla \log P(y_i)  \big].
\end{aligned}
\end{equation}
Next, we expand the log probability of the full sequence. Because the sequence probability is a product of stepwise conditional probabilities, its logarithm is a sum:
\begin{equation}
    \log P(y_i) = \log \left( \prod_{j=1}^{T_i} p_{ij} \right) = \sum_{j=1}^{T_i} \log p_{ij}.
\end{equation}
Substituting this expansion into our expected gradient yields the exact analytical policy gradient:
\begin{equation}
    \nabla q_x = \mathbb{E}_{y_i \sim \pi_\theta} \Bigg[ \sum_{j=1}^{T_i} r_i\nabla \log p_{ij} \Bigg].
\end{equation}

In practice, GRPO and similar reinforcement learning algorithms approximate this exact expectation using Monte Carlo sampling. By drawing a group of $G$ independent rollouts $\{y_1, y_2, \dots, y_G\}$ from the policy $\pi_\theta$, we replace the expectation with the empirical mean over the $G$ samples, yielding the final objective function gradient used to update the policy:
\begin{equation}
    \nabla q_x \approx \frac{1}{G} \sum_{i=1}^{G} \sum_{j=1}^{T_i} r_i\nabla \log p_{ij}.
\end{equation}
\end{proof}

\begin{proposition}
Consider a positive-only GRPO-style policy gradient method with binary rewards $r_i \in \{0,1\}$. In a positive-only setting, updates are restricted to successful trajectories, meaning the advantage is effectively masked by the reward: $A_{i}^+ = (r_i - \bar{r})r_i$, where $\bar{r} = \frac{1}{G}\sum_{k=1}^G r_k$ is the group mean reward. Assuming the empirical mean $\bar{r}$ closely approximates the true success probability $q_x$ (which we abbreviate as $q$), the gradient objective satisfies:
\begin{equation}\label{eq:adv_policy_grad}
    \frac{1}{D} \sum_{i=1}^{G} \sum_{j=1}^{T_i} A_{i}^+ \nabla \log p_{ij} \approx \alpha \cdot \nabla q_x,
\end{equation}
where the evaluated term $\alpha \cdot \nabla q_x$ corresponds to the choices of $D$ as follows:
\renewcommand{\arraystretch}{1.5}
\begin{center}
\begin{tabular}{|l|l|}
    \hline
        $D$ & Evaluated $\alpha \cdot \nabla q_x$ \\ \hline
        $G$ & $(1 - q) \nabla q_x$ \\ 
        $G \cdot \bar{r}$ & $\frac{1 - q}{q} \nabla q_x$ \\ \hline
\end{tabular}
\end{center}
\end{proposition}

\begin{proof}
Let $P(y_i) = \prod_{j=1}^{T_i} p_{ij}$ be the probability of generating the sequence $y_i$. By bringing the inner sum over $j$ inside the derivative and $\log$ we find
\begin{equation}
    \sum_{j=1}^{T_i} \nabla \log p_{ij} = \nabla \log P(y_i).
\end{equation}
Substituting this into the left-hand side (LHS) of \Cref{eq:adv_policy_grad} yields:
\begin{equation}
    \text{LHS} = \frac{1}{D} \sum_{i=1}^{G}  A_i^+ \nabla \log P(y_i).
\end{equation}
Because rewards are binary, the advantage only applies when $r_i = 1$, and we can rewrite $A_i^+ = (1 - \bar{r})r_i$. Substituting this simplified advantage back into our sum,
\begin{equation}
    \text{LHS} = \frac{1 - \bar{r}}{D} \sum_{i=1}^{G} r_i \nabla \log P(y_i).
\end{equation}
From Equation \ref{eq:policy-gradient-seq} in our previous proof, we established that the empirical estimate for the gradient of the success probability is $\nabla q_x \approx \frac{1}{G} \sum_{i=1}^{G} r_i\nabla \log P(y_i)$. Rearranging this relationship gives:
\begin{equation}
    \sum_{i=1}^{G} r_i \nabla \log P(y_i) \approx G \nabla q_x.
\end{equation}
Substituting this into our LHS equation yields
\begin{equation}
    \text{LHS} \approx \frac{G(1 - \bar{r})}{D} \nabla q_x.
\end{equation}
We now assume that the empirical mean reward over the group $G$ approximates the true probability of generating a correct answer, such that $\bar{r} \approx q$. Applying this approximation, we arrive at the general form:
\begin{equation}
    \text{LHS} \approx \frac{G(1 - q)}{D} \nabla q_x.
\end{equation}
Finally, we evaluate this expression for the two specific choices of the normalizer $D$:
\vspace{0.5em}

\noindent \textbf{Case 1:} $D = G$. \\
Substituting $D = G$ into the general form, the $G$ terms cancel out:
\begin{equation}
    \text{LHS} \approx (1 - q) \nabla q_x.
\end{equation}
This confirms the first row of the table.

\vspace{0.5em}

\noindent \textbf{Case 2:} $D = G \cdot \bar{r}$. \\
Using our approximation $\bar{r} \approx q$, we substitute $D \approx G \cdot q$:
\begin{equation}
    \text{LHS} \approx  \frac{1 - q}{q} \nabla q_x.
\end{equation}
This confirms the second row of the table, completing the proof. \hfill
\end{proof}

\section{Baseline Formulations}

Here we review the baselines that we compare WAPO against. Let $x$ be a prompt, $\{y_i\}_{i=1}^G$ the sampled completions for $x$, and $y_{i,t}$ the $t$-th token of completion $y_i$. We define the token-level importance ratio:
\begin{equation*}
\rho_{it}(\theta) = \frac{\pi_\theta(y_{i,t}\mid x,y_{i,<t})}{\pi_{\theta_{\mathrm{old}}}(y_{i,t}\mid x,y_{i,<t})}.
\end{equation*}
For \GRPO{} and \GSPO{}, advantages are group-normalized: 
\begin{align*}
A_i &= \frac{r_i-\bar r}{\operatorname{std}(\{r_j\}_{j=1}^G)+\epsilon}, \\
\bar r &= \frac{1}{G}\sum_{j=1}^G r_j. 
\end{align*}

\paragraph{\GRPO}
The \GRPO{} objective uses token-level clipped importance ratios: 
\begin{equation*}
\begin{aligned}
\mathcal{L}_{\GRPO}(\theta) &= -\frac{1}{G}\sum_{i=1}^G \frac{1}{|y_i|} \sum_{t=1}^{|y_i|} \min \Biggl( \rho_{it}(\theta) A_i, \\
&\quad \operatorname{clip}\bigl(\rho_{it}(\theta), 1-\epsilon, 1+\epsilon\bigr) A_i \Biggr).
\end{aligned}
\end{equation*}

\paragraph{\GSPO}
\GSPO{} replaces token-level ratios with a sequence-level ratio. In our implementation, the sequence ratio is the geometric mean of token ratios:
\begin{equation*}
\rho_i^{\mathrm{seq}}(\theta) = \exp\!\left( \frac{1}{|y_i|} \sum_{t=1}^{|y_i|} \log \rho_{it}(\theta) \right).
\end{equation*}
The objective is:
\begin{equation*}
\begin{aligned}
\mathcal{L}_{\GSPO}(\theta) &= -\frac{1}{G}\sum_{i=1}^G \frac{1}{|y_i|} \sum_{t=1}^{|y_i|} \min \Biggl( \rho_i^{\mathrm{seq}}(\theta) A_i, \\
&\quad \operatorname{clip}\bigl(\rho_i^{\mathrm{seq}}(\theta), 1-\epsilon, 1+\epsilon\bigr) A_i \Biggr).
\end{aligned}
\end{equation*}

\paragraph{\DAPO}
For \DAPO{}, advantages are mean-centered but not variance-normalized:
\begin{equation*}
A_i^{\DAPO} = r_i-\bar r.
\end{equation*}
The objective uses token-level ratios with asymmetric clipping:
\begin{equation*}
\begin{aligned}
&\mathcal{L}_{\DAPO}(\theta) = -\frac{1}{|\mathcal{G}|} \sum_{x\in\mathcal{G}} \frac{1}{\sum_{i=1}^G |y_i|} \sum_{i=1}^G\sum_{t=1}^{|y_i|} \min \Big( \\
&\qquad \rho_{it}(\theta) A_i^{\DAPO},  \operatorname{clip}\bigl(\rho_{it}(\theta), \epsilon_{\mathrm{low}}, \epsilon_{\mathrm{high}}\bigr) A_i^{\DAPO} \Big).
\end{aligned}
\end{equation*}

\section{Experiment Details}
All our experiments are conducted on two NVIDIA A6000 GPUs, where we use one GPU for vllm rollout generation and one GPU for policy updates.

\subsection{Math Environment}
\label{app:math-env}

For mathematical reasoning, each example consists of a problem \(x\) and a ground-truth answer \(a\). The model produces a response \(y\), which must contain both a reasoning segment and a final answer segment. We append the following instruction to each math problem:
\begin{quote}
\small
Put your reasoning inside \texttt{<think>...</think>} tags, then write your final answer as: \texttt{Answer: <your answer>}.
\end{quote}
The required response format is
\[
\begin{aligned}
y ={}& \texttt{<think>}~\text{reasoning}~\texttt{</think>} \\
& \texttt{Answer: }~\text{answer}.
\end{aligned}
\]
The parser extracts the text following \texttt{Answer:} as the model's final answer \(\hat{a}(y)\). The format indicator is
\[
I_{\mathrm{fmt}}(y)
=
\mathbf{1}\{y \in \mathcal{Y}_{\mathrm{fmt}}\},
\]
where \(\mathcal{Y}_{\mathrm{fmt}}\) is the set of responses matching the required format. The answer correctness indicator is
\[
I_{\mathrm{ans}}(y,a)
=
\mathbf{1}\{\hat{a}(y)=a\}.
\]
The math reward is binary:
\[
R_{\mathrm{math}}(y,a)
=
I_{\mathrm{fmt}}(y) I_{\mathrm{ans}}(y,a).
\]
A response receives reward 1 only when it follows the required format and the extracted final answer exactly matches the ground-truth answer; otherwise it receives reward 0.

\subsubsection{\datamath{} dataset.}
For the \datamath{} experiments, we use the augmented split released with PRM800K \citep{lightman2024let}. This split was constructed to reduce overfitting to the original 7,500-problem \datamath{} training set: 4,500 problems from the original \datamath{} test split are added to training, yielding 12,000 training problems, and the remaining 500 test problems are held out for evaluation. Each example contains a problem statement, reference solution, final answer, subject, difficulty level, and unique identifier. We use the problem statement and final answer only. Each problem is converted into a single-turn chat prompt. 

\subsubsection{\numina{} dataset.}
\label{app:numina-prep}
We start from the \numina{} training split \cite{aimo2025numinamathlean, wang2025kimina}. The raw split contains 104,155 examples. We filter examples before training as follows: keep only examples whose question type is either math-word-problem or; drop examples with missing or empty problem; drop examples with missing or empty answer; and keep only answers that are simple numeric strings, namely integers, decimals, or plain fractions matching forms such as $3, -2, 4.5, \text{or } 7/8$. This removes symbolic answers, interval answers, LaTeX fractions such as $\frac{1}{2}$, multiple-answer strings, and labels such as unknown.                        
After filtering, the dataset contains 21,251 examples: 18,273 math-word-problem examples and 2,978 multiple-choice-question examples. We shuffle the filtered dataset with a fixed seed. By default, the first 1,000 shuffled examples are used as evaluation data and the remaining 20,251 examples are used for training. 

\subsection{Multi-hop QA Environment}
\label{app:qa-env}

For retrieval-based multi-hop question answering, each example consists of a question \(x\), a gold answer \(a\), and access to a retrieval corpus. The model interacts with the environment over multiple turns. We adopted Search-R1\cite{jinsearch} style environment interaction, the model will conduct reasoning inside <think>...</think>, issue search query in <search>...</search> and eventually produce a final answer in <answer>...</answer>.

The system prompt used for multi-hop QA has more instructions compared to Search-R1 to prevent hallucinations. We set $max\_turns$ to 10 in all multihop QA training and evaluation.:
\begin{quote}
\small
You are a helpful assistant. \textbf{Your task:}

You need to answer complex questions by retrieving relevant information and reasoning step-by-step.

You \textbf{must} conduct reasoning inside \texttt{<think>} and \texttt{</think>} in each of your responses before you perform a search or give the final answer. The reasoning should include analysing the question, making a plan, and answering sub-questions by reasoning on existing content.

You should obtain knowledge by calling a search engine by \texttt{<search> query </search>} when you come up with a suitable query. The responses from the search engine will be the top search results and will be between the tags \texttt{<information>} and \texttt{</information>}.

Always ground your answer to the information the search engine returns. Do not use any information that is not returned by the search engine. You can call the search engine any number of times you want.

Once enough information is gathered, produce a precise and concise answer to the original question and wrap it in \texttt{<answer>...</answer>}.
\end{quote}
Search turns use:
\[
\begin{aligned}
y_t ={}& \texttt{<think>}~\text{reasoning}~\texttt{</think>} \\
& \texttt{<search>}~\text{query}~\texttt{</search>}.
\end{aligned}
\]
Final-answer turns use:
\[
\begin{aligned}
y_t ={}& \texttt{<think>}~\text{reasoning}~\texttt{</think>} \\
& \texttt{<answer>}~\text{answer}~\texttt{</answer>}.
\end{aligned}
\]
Retrieved evidence is inserted back into the conversation between \texttt{<information>} and \texttt{</information>} tags.

Let \(\hat{a}(y)\) be the extracted final answer. We normalize both prediction and gold answer by lowercasing, removing punctuation and articles, and collapsing whitespace. The F1 reward is computed over normalized answer tokens as
\[
R_{\mathrm{F1}}
=
\frac{2PR}{P+R},
\]
where \(P\) and \(R\) are token precision and recall between \(\hat{a}(y)\) and \(a\). If the prediction and gold answer have no shared normalized tokens, \(R_{\mathrm{F1}}=0\).

We also use a format reward \(R_{\mathrm{fmt}}(y)\), which checks whether assistant messages follow the required structured search or answer format. The final multi-hop QA reward is
\[
R_{\mathrm{QA}}(y,a)
=
\frac{0.3\,R_{\mathrm{fmt}}(y)
+
R_{\mathrm{F1}}(y,a)}{1.3}.
\]

\begin{table*}[t]
\centering
\caption{Training hyperparameters for math and multi-hop QA experiments. Batch size here means the number of rollouts we generated for one batch, micro-batch size means number of rollouts trained on one GPU.}
\label{tab:hyperparams}
\begin{minipage}[t]{0.48\textwidth}
\centering
\begin{tabular}{lc}
\toprule
\multicolumn{2}{c}{\textbf{Math}} \\
\midrule
Hyperparameter & Value \\
\midrule
LoRA rank & 8 \\
LoRA alpha & 32 \\
Learning rate & \(1\times 10^{-5}\) \\
Batch size & 256 \\
Micro-batch size & 1 \\
Mini-batch size & 64 \\
Rollouts per prompt & 8 \\
Max response length & 4096 \\
Max sequence length & 8192 \\
Max grad norm & 1.0 \\
\bottomrule
\end{tabular}
\end{minipage}
\hfill
\begin{minipage}[t]{0.48\textwidth}
\centering
\begin{tabular}{lc}
\toprule
\multicolumn{2}{c}{\textbf{Multi-hop QA}} \\
\midrule
Hyperparameter & Value \\
\midrule
Max Turns & 10 \\
Retriever Top k & 3 \\
LoRA rank & 8 \\
LoRA alpha & 32 \\
Learning rate & \(1\times 10^{-5}\) \\
Batch size & 256 \\
Micro-batch size & 1 \\
Mini-batch-size & 64 \\
Rollouts per prompt & 8 \\
Max response length & 1024 \\
Max sequence length & 8192 \\
Max grad norm & 1.0 \\
\bottomrule
\end{tabular}
\end{minipage}
\end{table*}

\subsubsection{OTT-QA}
We used a subset of 20,000 samples from the \ott{}~\cite{chen2020open} training dataset to generate high quality trajectories for supervised fine-tuning (SFT) \citep{ouyang2022sft}. Among the generated trajectories, we applied rejection sampling to keep trajectories with valid searches, correct formats, and correct answers. We constructed a RL dataset using a disjoint subset of 10,000 samples from the \ott{} training dataset. To ensure exploration during RL we ran 6 rollouts for each question on a Qwen-SFT model and only kept questions with accuracy between 0\% and 50\%, exclusive. All \ott{} checkpoints were evaluated on a held-out test set of 1000 samples.

We created a dense retriever over the \ott{} table and passage corpus. Each Wikipedia passage is represented by its normalized title plus passage text, while each table is serialized using its title, section title, introduction, and markdown-formatted table content. We did not apply a chunking strategy, and all documents were encoded with the BAAI/bge-m3 embedding model~\citep{bge-m3}.

\subsubsection{Hotpot-QA \& 2wiki}
Similar to \ott{}, we used a subset of 10,000 samples from HotpotQ~\citep{yang2018hotpotqa} to generate a SFT coldstart dataset with the rest of the data used for RL training. Hotpot-QA checkpoints are evaluated on a held-out test set that contains 1000 samples. In our OOD experiments, we evaluated Hotpot-QA checkpoints on a 1000 sample subset of the 2wiki~\citep{ho2020constructing} test dataset.

We follow the Search-R1~\cite{jinsearch} instructions to setup a retriever using the wiki-18~\cite{karpukhin2020dense} corpus, and an E5~\cite{wang2022text} embedding model for HotpotQA and 2wiki training and evaluation. 
\subsection{Hyperparameters}
\label{app:hyperparams}

Table~\ref{tab:hyperparams} summarizes the training hyperparameters. We use the same settings across different model checkpoints within each task family. We use a larger learning rate $1 \times 10^{-5}$ because \citet{schulman2025lora} proposed LoRA RL training needs $10 \times$ the learning rate comparing to full fine-tuning.

\subsection{Supplementary Experiments}

\label{app:supp_exp}
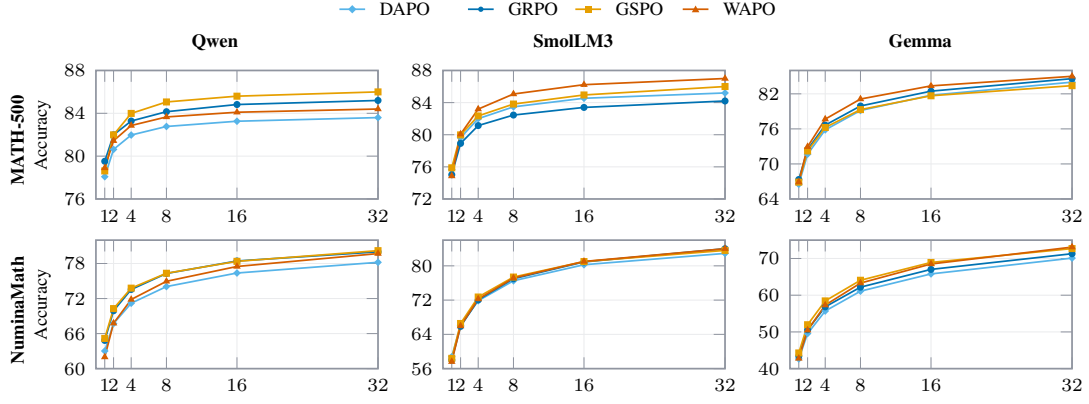
\begin{figure*}[!t]
\centering
\setlength{\tabcolsep}{0.6pt}
\renewcommand{\arraystretch}{0.74}
\begin{tikzpicture}[baseline=-0.5ex]
\draw[color=vfSky, line width=0.65pt] (0.00,0) -- (0.34,0);
\fill[color=vfSky] (0.17,0.055) -- (0.225,0) -- (0.17,-0.055) -- (0.115,0) -- cycle;
\node[anchor=west, font=\scriptsize] at (0.40,0) {DAPO};
\draw[color=vfBlue, line width=0.65pt] (1.70,0) -- (2.04,0);
\fill[color=vfBlue] (1.87,0) circle[radius=0.035];
\node[anchor=west, font=\scriptsize] at (2.10,0) {GRPO};
\draw[color=vfOrange, line width=0.65pt] (3.12,0) -- (3.46,0);
\fill[color=vfOrange] (3.25,-0.04) rectangle (3.33,0.04);
\node[anchor=west, font=\scriptsize] at (3.52,0) {GSPO};
\draw[color=vfRed, line width=0.65pt] (4.54,0) -- (4.88,0);
\fill[color=vfRed] (4.71,0.05) -- (4.66,-0.04) -- (4.76,-0.04) -- cycle;
\node[anchor=west, font=\scriptsize] at (4.94,0) {WAPO};
\end{tikzpicture}
\begin{tabular}{ccc}
\scriptsize\textbf{Qwen} & \scriptsize\textbf{SmolLM3} & \scriptsize\textbf{Gemma} \\
\llap{\raisebox{0.07\linewidth}{\rotatebox[origin=c]{90}{\scriptsize\textbf{MATH-500}}}\hspace{1mm}}%
\begin{tikzpicture}
\begin{axis}[wapo eval axis, xmin=0, xmax=32, ymin=76, ymax=88, ylabel={Accuracy}, ytick={76,80,84,88}, xtick={1,2,4,8,16,32}]
\addplot[wapo drdapo plot] coordinates {(1,78.08) (2,80.63) (4,81.97) (8,82.77) (16,83.26) (32,83.60)};
\addplot[wapo grpo plot] coordinates {(1,79.52) (2,81.99) (4,83.28) (8,84.15) (16,84.82) (32,85.20)};
\addplot[wapo gspo plot] coordinates {(1,78.62) (2,81.99) (4,83.99) (8,85.06) (16,85.60) (32,86.00)};
\addplot[wapo method plot] coordinates {(1,78.90) (2,81.44) (4,82.86) (8,83.66) (16,84.10) (32,84.40)};
\end{axis}
\end{tikzpicture} &
\begin{tikzpicture}
\begin{axis}[wapo eval axis, xmin=0, xmax=32, ymin=72, ymax=88, ylabel={}, ytick={72,76,80,84,88}, xtick={1,2,4,8,16,32}]
\addplot[wapo drdapo plot] coordinates {(1,75.66) (2,79.63) (4,82.00) (8,83.47) (16,84.54) (32,85.20)};
\addplot[wapo grpo plot] coordinates {(1,75.02) (2,78.92) (4,81.14) (8,82.45) (16,83.40) (32,84.20)};
\addplot[wapo gspo plot] coordinates {(1,75.88) (2,79.98) (4,82.34) (8,83.82) (16,84.95) (32,86.00)};
\addplot[wapo method plot] coordinates {(1,74.86) (2,80.07) (4,83.19) (8,85.07) (16,86.23) (32,87.00)};
\end{axis}
\end{tikzpicture} &
\begin{tikzpicture}
\begin{axis}[wapo eval axis, xmin=0, xmax=32, ymin=64, ymax=86, ylabel={}, ytick={64,70,76,82}, xtick={1,2,4,8,16,32}]
\addplot[wapo drdapo plot] coordinates {(1,66.50) (2,71.67) (4,75.83) (8,79.15) (16,81.76) (32,84.00)};
\addplot[wapo grpo plot] coordinates {(1,67.33) (2,72.47) (4,76.64) (8,79.92) (16,82.46) (32,84.60)};
\addplot[wapo gspo plot] coordinates {(1,66.88) (2,72.22) (4,76.24) (8,79.30) (16,81.67) (32,83.40)};
\addplot[wapo method plot] coordinates {(1,66.90) (2,72.95) (4,77.67) (8,81.11) (16,83.36) (32,85.00)};
\end{axis}
\end{tikzpicture} \\
\llap{\raisebox{0.07\linewidth}{\rotatebox[origin=c]{90}{\scriptsize\textbf{NuminaMath}}}\hspace{1mm}}%
\begin{tikzpicture}
\begin{axis}[wapo eval axis, xmin=0, xmax=32, ymin=60, ymax=82, ylabel={Accuracy}, ytick={60,66,72,78}, xtick={1,2,4,8,16,32}]
\addplot[wapo drdapo plot] coordinates {(1,63.02) (2,67.72) (4,71.20) (8,74.05) (16,76.37) (32,78.20)};
\addplot[wapo grpo plot] coordinates {(1,64.83) (2,69.91) (4,73.55) (8,76.29) (16,78.43) (32,80.00)};
\addplot[wapo gspo plot] coordinates {(1,65.17) (2,70.28) (4,73.78) (8,76.33) (16,78.37) (32,80.20)};
\addplot[wapo method plot] coordinates {(1,61.99) (2,67.77) (4,71.84) (8,74.95) (16,77.49) (32,79.70)};
\end{axis}
\end{tikzpicture} &
\begin{tikzpicture}
\begin{axis}[wapo eval axis, xmin=0, xmax=32, ymin=56, ymax=86, ylabel={}, ytick={56,64,72,80}, xtick={1,2,4,8,16,32}]
\addplot[wapo drdapo plot] coordinates {(1,58.79) (2,66.12) (4,71.85) (8,76.51) (16,80.26) (32,82.90)};
\addplot[wapo grpo plot] coordinates {(1,58.02) (2,65.86) (4,72.05) (8,76.99) (16,80.91) (32,84.00)};
\addplot[wapo gspo plot] coordinates {(1,58.40) (2,66.54) (4,72.72) (8,77.38) (16,81.00) (32,83.60)};
\addplot[wapo method plot] coordinates {(1,57.63) (2,65.96) (4,72.25) (8,77.15) (16,81.02) (32,84.00)};
\end{axis}
\end{tikzpicture} &
\begin{tikzpicture}
\begin{axis}[wapo eval axis, xmin=0, xmax=32, ymin=40, ymax=75, ylabel={}, ytick={40,50,60,70}, xtick={1,2,4,8,16,32}]
\addplot[wapo drdapo plot] coordinates {(1,42.65) (2,49.56) (4,55.69) (8,61.10) (16,65.80) (32,70.10)};
\addplot[wapo grpo plot] coordinates {(1,43.39) (2,50.59) (4,56.83) (8,62.19) (16,67.03) (32,71.30)};
\addplot[wapo gspo plot] coordinates {(1,44.27) (2,51.96) (4,58.45) (8,64.07) (16,68.92) (32,72.70)};
\addplot[wapo method plot] coordinates {(1,42.78) (2,50.47) (4,57.30) (8,63.30) (16,68.50) (32,73.10)};
\end{axis}
\end{tikzpicture} \\
\end{tabular}
\caption{\textbf{MATH-500 and \numina{} pass@$k$ curves.} The base model is omitted from this figure as its performance is substantially lower than the trained baselines, which would compress the $y$-axis range and make it harder to distinguish differences among the remaining curves.}
\label{fig:math_numina_passk_curves_no_base}
\end{figure*}

\begin{table}[t]
\centering
\scriptsize
\providecolor{vfRed}{HTML}{D55E00}
\providecommand{\xmark}{\(\boldsymbol{\times}\)}
\providecommand{\runcollapse}{\makebox[0pt][l]{\,\textcolor{vfRed}{\footnotesize\xmark}}}
\caption{
\textbf{Main in-distribution results.}
We report average@32 at the best checkpoint, using EM for \ott{}/\hotpot{} and accuracy for \datamath{}/\numina{}. A red \textcolor{vfRed}{\footnotesize\xmark} marks collapse, where reward falls near zero. N/A shows failure within 100 steps.}
\label{tab:main_results}
\begin{tabular}{llcccc}
\toprule
Dataset & Model & DAPO & GRPO & GSPO & WAPO \\
\midrule
\multirow{3}{*}{OTT-QA}
& Qwen & N/A\runcollapse & 20.70 & 27.57 & \textbf{37.44} \\
& Smol & \textbf{32.63} & 31.27 & 11.76 & 29.59 \\
& Gemma & 4.67 & 6.62 & 10.50 & \textbf{13.67} \\
\midrule
\multirow{3}{*}{HotpotQA}
& Qwen & 16.56\runcollapse & 30.69 & 41.30 & \textbf{45.84} \\
& Smol & 42.01\runcollapse & 42.44 & 38.66 & \textbf{43.62} \\
& Gemma & 19.72\runcollapse & 21.09 & 24.36 & \textbf{34.91} \\
\midrule
\multirow{3}{*}{MATH-500}
& Qwen & 78.01 & \textbf{79.51} & 78.62 & 78.90 \\
& Smol & 75.66 & 75.01 & \textbf{75.88} & 74.86 \\
& Gemma & 66.50\runcollapse & \textbf{67.24} & 66.88 & 66.90 \\
\midrule
\multirow{3}{*}{NuminaMath}
& Qwen & 63.02 & 64.83 & \textbf{65.17} & 61.99 \\
& Smol & \textbf{58.79} & 58.02 & 58.40 & 57.63 \\
& Gemma & 42.65\runcollapse & 43.39\runcollapse & \textbf{44.27} & 42.78 \\
\bottomrule
\end{tabular}
\end{table}

Table~\ref{tab:main_results} summarizes the main in-distribution results. On multi-hop QA, \method{} achieves the strongest performance across model families, while \DAPO{} collapses in 4 out of 6 runs. Notably, \DAPO{} collapses within 100 steps on \ott{} with \qwen{}, so we do not report its score. On math benchmarks, \method{} is competitive but trails \GRPO{} or \GSPO{} on some \qwen{} and \smol{} runs, while consistently outperforming \DAPO{}. Additionally we include pass@$k$ curves on the math datasets in Figure \ref{fig:math_numina_passk_curves_no_base} where we see competitive performance in most settings. Overall, \method{} offers the strongest stability-performance tradeoff across tasks and model families.

\end{document}